\documentclass[letterpaper, 10 pt, conference]{ieeeconf}
\IEEEoverridecommandlockouts
\usepackage[hyphens]{url}
\usepackage{graphicx}
\usepackage{xcolor}
\makeatletter\let\NAT@parse\undefined\makeatother  
\usepackage[numbers,sort&compress]{natbib}
\usepackage{hyperref}
\usepackage{amsmath,amssymb,amsfonts}

\usepackage{amsthm}
\usepackage{mathtools}
\usepackage{bm}

\usepackage{algorithm}
\usepackage{algorithmic}

\let\labelindent\relax
\usepackage{enumitem}
\usepackage{booktabs}
\usepackage{multirow}

\newcommand{\R}{\mathbb{R}}
\newcommand{\cM}{\mathcal{M}}
\newcommand{\cG}{\mathcal{G}}
\newcommand{\cV}{\mathcal{V}}
\newcommand{\cH}{\mathcal{H}}

\newcommand{\diag}{\operatorname{Diag}}
\newcommand{\tr}{\operatorname{tr}}
\newcommand{\Had}{\circ}
\newcommand{\ip}[2]{\langle #1, #2 \rangle}

\theoremstyle{plain}
\newtheorem{proposition}{Proposition}

\theoremstyle{definition}
\newtheorem{definition}{Definition}
\theoremstyle{remark}
\newtheorem{remark}{Remark}

\newcommand{\BMorig}[1]{{#1}}

\title{\LARGE \bf Riemannian Optimization for\\Hadamard Products of Low-Rank Matrices}
\author{Pratik Jawanpuria$^{1}$, Ankish Chandresh$^{1}$, and Bamdev Mishra$^{2}$%
\thanks{$^{1}$Centre for Machine Intelligence and Data Science, Indian Institute of Technology Bombay, India {\tt\small pratik.jawanpuria@iitb.ac.in}, {\tt\small ankish.chandresh@iitb.ac.in}}%
\thanks{$^{2}$Microsoft India {\tt\small bamdevm@microsoft.com}}%
}

\begin{document}
\maketitle
\thispagestyle{empty}
\pagestyle{empty}

\begin{abstract}
The elementwise Hadamard product of two low-rank matrices provides a parameter-efficient model for data with multiplicative structure, but its modeling is challenging due to the presence of additional symmetries under coupled row/column scalings between the two factors. In order to leverage the geometry of the space, we formulate the learning of such matrices as optimization on a Riemannian quotient manifold. We propose a novel block-diagonal Riemannian metric derived from the pullback of the Frobenius inner product. The metric is shown to be invariant under these symmetries. We develop a Riemannian gradient descent algorithm that uses a tuning-free Gauss--Newton step size and scales linearly in the number of observed entries per iteration. The versatile framework of Riemannian quotient optimization enables both first-order and second-order Riemannian methods, the latter through a closed-form connection and the Riemannian Hessian. Experiments on real and synthetic datasets illustrate the efficacy of our proposed Riemannian approach.

\end{abstract}

\section{Introduction}
\label{sec:intro}

Low-rank matrix models are central to modern machine learning, underpinning everything from classical recommender systems~\cite{koren2009matrix} to representation learning~\cite{mikolov2013distributed} and the parameter-efficient fine-tuning of large-scale models~\cite{hu2022lora}.
However, in many complex domains, the underlying data structure is best captured not by a single low-rank matrix, but by the elementwise (Hadamard) product of two low-rank factors~\cite{ciaperoni2024hadamard}. For instance, in collaborative filtering, one factor pair can model overall user/item activity levels while the other captures nuanced preference structures~\cite{koren2009matrix}. Similar decoupled interactions arise in link prediction -- where community structure is separated from node-level heterogeneity~\cite{hoff2002latent} -- and in count data modeling for single-cell genomics~\cite{townes2019glmpca} and topic models~\cite{lee1999nmf}. The same Hadamard product of low-rank factors also underlies recent parameter-efficient fine-tuning of large models, where it adapts weights more expressively than a single low-rank update at an equal parameter budget~\cite{singhal2025abba,yu2025boha}.
Formally, we consider the model matrix $W$ given by:
\begin{equation}\label{eq:hadamard_form}
  W = (GH^\top) \Had (UV^\top),
\end{equation}
where $G\in\R^{m\times r_1}$, $H\in\R^{n\times r_1}$, $U\in\R^{m\times r_2}$, $V\in\R^{n\times r_2}$, and $\Had$ denotes the Hadamard (elementwise) product.
\BMorig{For brevity, we write $X = GH^\top$ and $Y = UV^\top$ for the two low-rank factors, so that $W = X\Had Y$.}


The Hadamard product of a rank-$r_1$ and a rank-$r_2$ matrix can have rank as high as $r_1 r_2$~\cite{ciaperoni2024hadamard}. The notion of Hadamard rank makes this expressivity precise~\cite{antolini2025hadamardranks}. A standard factorization $W = LR^\top$ achieving this rank would require $(m{+}n)(r_1r_2)$ parameters, whereas the Hadamard form~\eqref{eq:hadamard_form} requires only $(m{+}n)(r_1{+}r_2)$, a reduction \BMorig{when $(r_1{-}1)(r_2{-}1)>1$}. 

However, this structural efficiency introduces multiple optimization challenges. The Hadamard product induces symmetries beyond the standard general linear (GL) reparameterization invariance of fixed-rank factorizations. Specifically, alongside the GL invariance of each factor pair, the product $W$ is invariant under coupled row and column scalings between the two factors. The additional degrees of freedom result in a challenging optimization landscape, especially for standard Euclidean gradient methods.\\



\noindent \textbf{The limits of existing solvers.} \citet{ciaperoni2024hadamard} proposed alternating gradient descent (AGD) for the model~\eqref{eq:hadamard_form}, but its Euclidean steps ignore the per-row scaling and the curved geometry of the factor space. \citet{wertz2024bcd} proposed block coordinate descent (BCD), which updates each factor by solving per-row normal equations in closed form. BCD is highly effective for least squares, yet for the non-quadratic losses central to robust statistics, logistic link prediction, and cross-entropy objectives its sub-problems have no closed form. AGD applies to any smooth loss but forfeits the curvature adaptation that makes BCD effective.




\noindent \textbf{Our approach based on Riemannian optimization.}
In this paper, we bridge this gap by elevating the problem of learning Hadamard product of low-rank matrices to an optimization problem on a Riemannian manifold. To this end, we propose a Riemannian quotient manifold that takes into account the invariances that exist naturally in the Hadamard product of low-rank matrices (\ref{eq:hadamard_form}).
Exploiting the Riemannian structure, we discuss the required ingredients to enable optimization \citep{boumal2023introduction, absilbook}. 
Concretely, our contributions are:
\begin{enumerate}[nosep]
\item \textbf{Geometric characterization:} We formally describe the quotient manifold structure of $W = (GH^\top)\Had(UV^\top)$, identifying both the GL reparameterization and the novel cross-Hadamard scaling symmetries (Section~\ref{sec:geometry}).
\item \textbf{A novel Riemannian metric:} We propose a block-diagonal Riemannian metric that is scale invariant (Section~\ref{sec:geometry}). We also discuss the relationship between our proposed approach and BCD's normal-equation matrices that are used in its updates~\cite{wertz2024bcd}.

\item \textbf{Riemannian gradient descent (RGD):} We derive an RGD algorithm equipped with a Gauss-Newton Cauchy step size that requires no step-size tuning, achieving an $O(Nr + (m{+}n)r^4)$ per-iteration cost, where $N$ is the number of observed entries (dataset size in learning problems) and $r = \max(r_1,r_2)$ ensuring linear scaling with $N$.

\item \textbf{Riemannian trust-region method (RTR):} From the same metric we derive a closed-form Levi--Civita connection, and Riemannian Hessian expression allowing to use Riemannian trust-region method (Sections~\ref{sec:rtr} and~\ref{sec:exp_secondorder}).

\item \textbf{Empirical validation:} We demonstrate experimentally that RGD performs comparably to BCD on least-squares tasks (across three real datasets) and that RGD achieves the lowest RMSE on MovieLens-1M across all tested Hadamard rank configurations, and we show that our trust-region method matches the concurrent second-order method of \citet{gillis2026manifold} on its own benchmarks (Section~\ref{sec:experiments}).
\end{enumerate}

\section{Background and related work}
\label{sec:related}

\noindent\textbf{Block coordinate descent (BCD)~\cite{wertz2024bcd}.} 
For the least-squares loss $\tfrac{1}{2}\|A - W\|_F^2$ with $W$ modeled as in~\eqref{eq:hadamard_form}, each BCD sub-problem is a row-by-row least-squares system. The update for row $i$ of $G$ is
\begin{equation}
  G_{i,:} \leftarrow \arg\min_{g\in\R^{r_1}} \tfrac{1}{2}\bigl\|\bigl(A_{i,:} - g(H\diag(Y_{i,:}))^\top\bigr)\bigr\|^2,
  \label{eq:bcd_sub}
\end{equation}
with $Y = UV^\top$ and normal-equation matrix $H^\top\diag(Y_{i,:}^2)H \in \R^{r_1\times r_1}$, which encodes how the current $H$ and $Y$ shape the update for $G$. This closed form is bound to the squared-error loss and does not exist for objectives such as cross-entropy, rendering BCD inapplicable there. AGD~\cite{ciaperoni2024hadamard} instead handles any smooth loss but uses uniform Euclidean steps that ignore this factor-dependent geometry.\\

\noindent \textbf{Riemannian optimization on quotient manifolds.} 
Riemannian optimization provides a principled framework for unconstrained optimization over smooth manifolds~\cite{absilbook, boumal2023introduction}. A Riemannian solver requires three core components: a smooth manifold $\mathcal{M}$, a Riemannian metric $g_x$ on the tangent spaces to define gradients, and a retraction to map steps back to $\mathcal{M}$. 

For the fixed-rank manifold with factorizations $W = LR^\top$, the quotient geometry under $\text{GL}(r)$ invariance is well studied~\cite{vandereycken2013low, mishra2014r3mc}. The choice of the Riemannian metric is critical. While the Euclidean metric ignores factorization symmetries, a full pullback metric captures all curvature at the cost of coupling factor blocks and yielding dense gradient systems. A practical choice is to retain only the block-diagonal elements of the pullback metric~\cite{mishra2012riemannian}.

In particular, \citet{mishra2016riemannian} showed that tailoring the Riemannian metric to the problem's underlying structure--such as utilizing normal-equation matrices--acts as a  preconditioner for the optimization, significantly accelerating convergence over Euclidean metrics. Generic implementations of these concepts exist in solvers like Manopt~\cite{manopt}, Pymanopt~\cite{pymanopt}, and Manopt.jl~\cite{manoptjl}. Our work extends this quotient manifold framework beyond a single low-rank factor to the Hadamard product of two low-rank factors. 

\noindent\textbf{Concurrent work.} Concurrently with our work, \citet{gillis2026manifold} also study manifold optimization for the Hadamard decomposition. They reformulate the problem as $X \approx WH^\top$ with $W$ and $H$ constrained to product manifolds, and they develop a manifold based method, a block projected gradient method, and a projection-free manifold gradient descent, together with new initializations for large sparse least-squares problems. Our approach differs in four respects. First, we form a quotient manifold that removes both the general-linear reparameterization and the cross-Hadamard row and column scaling symmetries. Second, we design a scale-invariant block-diagonal metric that reuses BCD's normal-equation matrices as a preconditioner. Third, we obtain a tuning-free Gauss--Newton Cauchy step size. Fourth, our framework handles any smooth loss, whereas their algorithms target the least-squares objective. 
\section{Manifold Geometry and Metric}
\label{sec:geometry}

We consider optimization problems of the form
\begin{equation*}
  \min_{G,H,U,V} \; f\!\big(\Phi(G,H,U,V)\big),
\end{equation*}
where $\Phi(G,H,U,V) \triangleq (GH^\top)\Had(UV^\top) = W$ and $f\colon\R^{m\times n}\to\R$ is any smooth loss.
All proofs are in Appendix~\ref{app:proofs}. Appendix~\ref{app:secondorder} discusses the second-order ingredients like the Levi--Civita connection.


\subsection{Symmetries and quotient structure}

The factorization $X = GH^\top$ is invariant under $(G,H)\mapsto(GQ_1, HQ_1^{-\top})$ for any $Q_1\in\mathrm{GL}(r_1)$, and analogously $(U,V)\mapsto(UQ_2, VQ_2^{-\top})$ for $Q_2\in\mathrm{GL}(r_2)$.
These \emph{reparameterization symmetries} mean that the factors are defined only up to invertible linear transformations.

In addition, the Hadamard product $W = X\Had Y$ is invariant under coupled row/column scalings:
\begin{equation}
  (X,Y)\mapsto(D_r X D_c,\; D_r^{-1}YD_c^{-1})
  \label{eq:cross_scaling}
\end{equation}
for invertible diagonal $D_r\in\R^{m\times m}$, $D_c\in\R^{n\times n}$.
This \emph{cross-Hadamard scaling symmetry} introduces $m+n$ additional degrees of freedom that a good metric should respect.

Therefore, a natural symmetry group is the direct product $\cG = \bigl(\mathrm{GL}(r_1)\times\mathrm{GL}(r_2)\bigr)\times\bigl(\mathrm{Diag}_*(m)\times\mathrm{Diag}_*(n)\bigr)$,
where $\mathrm{Diag}_*(k)$ denotes invertible $k\times k$ diagonal matrices. Note that we call $\cG$ a natural (rather than the full) symmetry group: it is generated by the GL reparameterizations and the cross-Hadamard scalings, while the Hadamard model may admit further non-identifiabilities that we do not exploit.
Since the reparameterization and scaling actions commute (\textrm{GL} acts within each factor pair, while diagonal scaling acts across pairs), our quotient manifold is
\begin{equation}
  \cM = \bigl(\R_*^{m\times r_1}\!\times\R_*^{n\times r_1}\times\R_*^{m\times r_2}\!\times\R_*^{n\times r_2}\bigr)\big/\cG,
\end{equation}
where $\R_*^{p\times r}$ denotes full-column-rank $p\times r$ matrices. The action of $\cG$ is not free: the one-parameter subgroup $(Q_1,Q_2,D_r,D_c) = (\lambda I_{r_1},\, \lambda^{-1} I_{r_2},\, \lambda^{-1} I_m,\, \lambda I_n)$ with $\lambda\neq 0$ fixes every representative $(G,H,U,V)$, so the effective quotiented group has dimension $r_1^2 + r_2^2 + m + n - 1$. In particular, the global rescaling $X\mapsto\alpha X$, $Y\mapsto\alpha^{-1}Y$ is the cross-Hadamard scaling with $D_r=D_c=\sqrt{\alpha}\,I$, and its overlap with the $\mathrm{GL}$ scalar reparameterizations produces this one-dimensional stabilizer. The dimension of the proposed manifold $\cM$ is $\dim(\cM) = (m{+}n)(r_1{+}r_2) - r_1^2 - r_2^2 - m - n {}+ 1$. In the following, we represent a point on the manifold $\cM$ by $x = (G,H,U,V)$.

\paragraph{Model map and its differential}
The differential of $\Phi$ at $x$ applied to a tangent vector $\eta = (\eta_G,\eta_H,\eta_U,\eta_V)$ is
\begin{equation}
  D\Phi_x[\eta] = \Delta X[\eta]\Had Y + X\Had\Delta Y[\eta],
  \label{eq:dphi}
\end{equation}
where $\Delta X[\eta] = \eta_G H^\top + G\eta_H^\top$ and $\Delta Y[\eta] = \eta_U V^\top + U\eta_V^\top$.

\paragraph{Vertical space}
The vertical space at $(G,H,U,V)$ consists of all tangent directions that leave $W = \Phi(x)$ unchanged:
\begin{align}
  \cV_x = \{\,(&G\Lambda_1 + \diag(\delta_r)G,\; -H\Lambda_1^\top + \diag(\delta_c)H, \notag\\
               &U\Lambda_2 - \diag(\delta_r)U,\; -V\Lambda_2^\top - \diag(\delta_c)V) \notag\\
               &: \Lambda_k\in\R^{r_k\times r_k},\; \delta_r\in\R^m,\; \delta_c\in\R^n \,\}.
\end{align}
The horizontal space $\cH_x = \cV_x^{\perp_g}$ is the metric-orthogonal complement. Horizontal vectors represent tangent directions on the quotient $\cM$, i.e., directions that change $W$ and are thus meaningful for optimization. The metric $g_x$, restricted to horizontal vectors, induces a well-defined Riemannian metric on the quotient $\cM$~\cite{absilbook}.

\subsection{Proposed scale-invariant Riemannian metric}

A Riemannian metric on the factor space must satisfy two requirements:
\begin{enumerate}[nosep,label=(\roman*)]
  \item \textbf{GL-invariance:} invariant under the reparameterization group action $(G,H)\mapsto(GQ_1,HQ_1^{-\top})$, so that $g_x$ descends to a well-defined metric on the quotient $\cM$.
  \item \textbf{Cross-Hadamard scaling invariance:} It should be invariant under the coupled row/column scalings~\eqref{eq:cross_scaling}, so that the optimizer is insensitive to the (arbitrary) magnitude split between $X$ and $Y$.
\end{enumerate}
We begin with the \emph{pullback of the Frobenius inner product} through the model map:
\begin{equation}
  \ip{\eta}{\zeta}_{\mathrm{pb}} \triangleq \ip{D\Phi_x[\eta]}{D\Phi_x[\zeta]}_F.
  \label{eq:pullback}
\end{equation}
Expanding~\eqref{eq:pullback} via~\eqref{eq:dphi} yields four block-diagonal terms (one per factor) plus cross terms coupling $(G,H)$ with $(U,V)$ and coupling each factor with its pair.
The cross terms couple all four factor blocks, making the resulting linear system dense.
We therefore propose a metric that retains the diagonal blocks of the pullback while discarding the cross terms:

\begin{definition}[Hadamard-weighted metric]\label{def:metric}
For $x=(G,H,U,V)$ with full-column-rank factors and tangent vectors $\eta,\zeta$, define
\begin{align}
  g_x(\eta,\zeta)
  &= \sum_{i=1}^m \eta_{G,i:} S_i^G \zeta_{G,i:}^\top
   + \sum_{j=1}^n \eta_{H,j:} S_j^H \zeta_{H,j:}^\top \notag\\
  &\quad + \sum_{i=1}^m \eta_{U,i:} S_i^U \zeta_{U,i:}^\top
   + \sum_{j=1}^n \eta_{V,j:} S_j^V \zeta_{V,j:}^\top,
  \label{eq:metric}
\end{align}
where $\eta_{G,i:}$ denotes row $i$ of $\eta_G$ (and similarly for other factors), with $r\times r$ weight matrices
\begin{align}
  S_i^G &= H^\top\diag(Y_{i,:}^2)H, &
  S_j^H &= G^\top\diag(Y_{:,j}^2)G,
  \label{eq:sblocks_GH} \\
  S_i^U &= V^\top\diag(X_{i,:}^2)V, &
  S_j^V &= U^\top\diag(X_{:,j}^2)U.
  \label{eq:sblocks_UV}
\end{align}
\end{definition}
\noindent\textbf{Regularity convention.} Throughout, we work on the \emph{regular subset} $\cM_* \subseteq \cM$, the open dense set on which the factors are full column rank and $X = GH^\top$, $Y = UV^\top$ have no zero entries. On $\cM_*$ every weight block~\eqref{eq:sblocks_GH}--\eqref{eq:sblocks_UV} is positive definite (Proposition~\ref{prop:metric}), so $g_x$ is a valid Riemannian metric, and all metric, algorithmic, and convergence statements below are understood to hold on $\cM_*$ without further comment.
Equivalently, $g_x$ can be written in compact matrix form as
\begin{align}
  g_x(\eta,\zeta)
  &= \tr\bigl((\eta_G H^\top\Had Y)^\top(\zeta_G H^\top\Had Y)\bigr) \notag\\
  &+ \tr\bigl((\eta_H G^\top\Had Y^\top)^\top(\zeta_H G^\top\Had Y^\top)\bigr) \notag\\
  &+ \tr\bigl((\eta_U V^\top\Had X)^\top(\zeta_U V^\top\Had X)\bigr) \notag\\
  &+ \tr\bigl((\eta_V U^\top\Had X^\top)^\top(\zeta_V U^\top\Had X^\top)\bigr),
  \notag
\end{align}
which makes explicit that each term measures the squared Frobenius norm of one factor's contribution to $D\Phi_x[\eta]$, weighted by the complementary Hadamard factor.

\begin{proposition}[Properties of $g_x$]\label{prop:metric}
The bilinear form~\eqref{eq:metric} satisfies:
\begin{enumerate}[nosep,label=(\alph*)]
\item \textbf{Positive definiteness:} $g_x(\eta,\eta) > 0$ for $\eta \neq 0$ whenever the factors are full column rank and \BMorig{neither $X$ nor $Y$ has zero entries}.
\item \textbf{$\cG$-invariance:} $g_x$ is invariant under both GL reparameterization and cross-Hadamard scaling, stated precisely and proved in Proposition~\ref{prop:gauge}.
\end{enumerate}
\end{proposition}

Note that the S-blocks~\eqref{eq:sblocks_GH}-\eqref{eq:sblocks_UV}, and hence the proposed metric, depend only on the current factors $(G,H,U,V)$, not on the loss $f$.

\noindent \textbf{Why this particular approximation?}
The full pullback~\eqref{eq:pullback} is the Gauss--Newton form $J^\top J$ with $J = D\Phi_x$, whose cross terms couple all four factor blocks and make the linear system dense. Retaining only its block-diagonal part gives each factor row an independent $r\times r$ weight matrix $S$, so the Riemannian gradient reduces to independent $r\times r$ solves per row and column (Section~\ref{sec:algorithms}). The retained blocks still encode the local curvature of the Hadamard map, so $g_x$ acts as a block-diagonal Gauss--Newton preconditioner rather than the exact pullback metric, trading a modest amount of curvature information.



We now verify that this metric satisfies both invariance requirements.

\begin{proposition}[$\cG$-invariance]\label{prop:gauge}
The metric $g_x$ is invariant under the symmetry group $\cG$:
\begin{enumerate}[nosep,label=(\alph*)]
\item \textbf{GL reparameterization:} $(G,H,U,V)\mapsto(GQ_1, HQ_1^{-\top}, UQ_2, VQ_2^{-\top})$ for any $(Q_1,Q_2)\in\mathrm{GL}(r_1)\times\mathrm{GL}(r_2)$.
\item \textbf{Cross-Hadamard scaling:} $(G,H,U,V)\mapsto(D_rG, D_cH, D_r^{-1}U, D_c^{-1}V)$ for any invertible diagonal $D_r$, $D_c$.
\end{enumerate}
\end{proposition}

\noindent \textbf{Connection to BCD.}
The weight blocks $S_i^G$ are exactly BCD's normal-equation matrices~\eqref{eq:bcd_sub}.
The Riemannian gradient (derived in Section~\ref{sec:algorithms}) can therefore be interpreted as performing ``one soft BCD step'' on every row simultaneously, but for an arbitrary smooth loss function rather than least squares alone.

\section{Riemannian Optimization Algorithms}
\label{sec:algorithms}

Using this metric, we derive the Riemannian gradient and develop a Riemannian gradient descent (RGD) algorithm.
The entire algorithm interacts with the loss function $f$ only through the $W$-space gradient $E = \nabla_W f \in \R^{m\times n}$. The manifold machinery (metric, S-blocks, retraction) is completely independent of the loss function $f$.

\subsection{Euclidean and Riemannian gradients}

Given $E = \nabla_W f$, define $E_Y = E\Had Y$ and $E_X = E\Had X$. By the chain rule, the Euclidean factor gradients are
\begin{equation}
\begin{array}{ll}
  \nabla_G f = E_Y H, \nabla_H f = E_Y^\top G, \\
  \nabla_U f = E_X V, \nabla_V f = E_X^\top U.
  \end{array}
  \label{eq:egrad}
\end{equation}
Under the metric~\eqref{eq:metric}, the Riemannian gradient $\mathrm{grad}\,f(x)$ is the unique tangent satisfying $Df(x)[\xi] = g_x(\mathrm{grad}\,f(x),\xi)$ for all $\xi$.
Since $g_x$ is block-diagonal with blocks indexed by rows/columns, this identity decouples into independent $r\times r$ systems.
For the $G$-block, row $i$: $(\nabla_G f)_{i:}\xi_{G,i:}^\top = (\mathrm{grad}_G f)_{i:} S_i^G \xi_{G,i:}^\top$ for all $\xi_{G,i:}$, which gives
\begin{align}
  (\mathrm{grad}_G f)_{i:} &= (\nabla_G f)_{i:}\,(S_i^G)^{-1}, \label{eq:rgrad_G}\\
  (\mathrm{grad}_H f)_{j:} &= (\nabla_H f)_{j:}\,(S_j^H)^{-1}, \label{eq:rgrad_H}
\end{align}
and analogously for $U$, $V$.
Each S-block solve is an $r\times r$ linear system. Since $f$ depends on $x$ only through $\Phi(x)$ and the metric $g_x$ is $\cG$-invariant, the Riemannian gradient~\eqref{eq:rgrad_G}--\eqref{eq:rgrad_H} automatically lies in the horizontal space $\cH_x$~\cite{absilbook, boumal2023introduction}.

\begin{remark}[Cheap gradient norm]\label{rem:cheap_norm}
The squared Riemannian gradient norm satisfies
$\|\mathrm{grad}\,f\|_x^2 = g_x(\mathrm{grad}\,f,\mathrm{grad}\,f) = Df(x)[\mathrm{grad}\,f] = \sum_F \tr(\nabla_F f^\top\,\mathrm{grad}_F f)$,
where the second equality uses the defining relation of the Riemannian gradient and the sum runs over the four factors $F\in\{G,H,U,V\}$.
This is a plain Euclidean inner product, avoiding S-block recomputation.
\end{remark}

\subsection{Retraction}
A retraction on the total space induces a well-defined retraction on the quotient manifold~\cite[Prop.~4.1.3]{absilbook}.
We use the additive retraction:
\begin{equation}
  R_x(\eta) = (G{+}\eta_G,\; H{+}\eta_H,\; U{+}\eta_U,\; V{+}\eta_V),
  \label{eq:retraction}
\end{equation}
which satisfies the retraction axioms $R_x(0) = x$ and $DR_x(0)[\eta] = \eta$ by inspection.
It is valid on the open set of full-column-rank factors.

\subsection{Cauchy step size}
\label{sec:cauchy}

The retraction $R_x$ appearing below maps a tangent step back to the manifold and is defined in~\eqref{eq:retraction}.
RGD uses Armijo back-tracking for the step size: starting from an initial trial step $\alpha_k^0$, the step is contracted by a factor $\tau$ (we use $\tau{=}0.5$) until the sufficient-decrease condition $f(R_x(\alpha\,d)) \leq f(x) + c_1\alpha\,Df(x)[d]$ is satisfied (we use $c_1{=}10^{-4}$) \cite{absilbook,boumal2023introduction}.
We derive the initial trial step from a \emph{Gauss--Newton approximation}.
Linearizing $\Phi$ along the retracted curve $\alpha \mapsto R_x(\alpha\,d)$ and replacing the loss Hessian $\nabla^2_W f$ with the identity (the standard Gauss--Newton simplification, which is exact for the least-squares loss) yields the quadratic model
\begin{equation}
  f(R_x(\alpha\,d)) \approx f(x) + \alpha\,Df(x)[d] + \tfrac{\alpha^2}{2}\|D\Phi_x[d]\|_F^2.
  \label{eq:cauchy_model}
\end{equation}
Minimizing over $\alpha$ gives the Cauchy step:
\begin{equation}
  \alpha^{\mathrm{Cauchy}} = \frac{-Df(x)[d]}{\|D\Phi_x[d]\|_F^2}.
  \label{eq:cauchy}
\end{equation}
Using~\eqref{eq:dphi} with $\Delta X[d] = d_G H^\top + G\,d_H^\top$ and $\Delta Y[d] = d_U V^\top + U\,d_V^\top$, the numerator and denominator in (\ref{eq:cauchy}) are
\begin{align}
  Df(x)[d] &= \ip{E}{\Delta X[d]\Had Y + X\Had\Delta Y[d]}_F, \notag\\
  \|D\Phi_x[d]\|_F^2 &= \|\Delta X[d]\Had Y + X\Had\Delta Y[d]\|_F^2. \notag
\end{align}
The numerator can also be evaluated cheaply via Remark~\ref{rem:cheap_norm} as $Df(x)[d] = -\sum_F \tr(\nabla_F f^\top\,\mathrm{grad}_F f)$ when $d = -\mathrm{grad}\,f$.
The Cauchy step is invariant under the symmetry group $\cG$: since $g_x$ is $\cG$-invariant (Proposition~\ref{prop:gauge}), the search direction $d = -\mathrm{grad}\,f(x)$ transforms equivariantly under $\cG$, and because $\Phi$ is invariant under $\cG$, its differential satisfies $D\Phi_{\tilde x}[\tilde d] = D\Phi_x[d]$ for any equivalent representative $(\tilde x, \tilde d)$.
Thus both the numerator $Df(x)[d] = \ip{E}{D\Phi_x[d]}_F$ and the denominator $\|D\Phi_x[d]\|_F^2$ are invariant under $\cG$, so $\alpha^{\mathrm{Cauchy}}$ is independent of the choice of representative.
For non-quadratic losses, \eqref{eq:cauchy_model} is an approximation; the Armijo back-tracking corrects for any discrepancy.
This step size is positive (since $d$ is a descent direction, $Df(x)[d] < 0$) and requires no step-size tuning.

The Armijo line search ensures sufficiently small steps in practice. Because $\cM_*$ is open and the Armijo condition enforces a monotone decrease of $f$, a sufficiently small step keeps each iterate regular. In practice, we accept a trial step only if the retracted factors stay full rank with $X$ and $Y$ free of near-zero entries, and we contract further otherwise. Under these safeguards, standard Riemannian descent theory~\cite{absilbook,boumal2023introduction} applies. The objective decreases monotonically, and every accumulation point of the iterates is a stationary point, that is, $\mathrm{grad}\,f(x_k)\to 0$.

\subsection{Algorithms}

\subsubsection{Riemannian gradient descent (RGD)}

\begin{algorithm}[t]
\caption{RGD on the Hadamard low-rank manifold}
\label{alg:rgd}
\begin{algorithmic}[1]
\STATE \textbf{Input:} smooth loss $f$, ranks $r_1,r_2$
\STATE Initialize factors $G_0,H_0,U_0,V_0$ (SVD-based)
\FOR{$k=0,1,2,\dots$ \textbf{until} $\|\mathrm{grad}\,f\|_x / \|\mathrm{grad}\,f\|_{x_0} < \mathrm{tol}$}
    \STATE $x_k=(G_k,H_k,U_k,V_k)$.
  \STATE $X_k \leftarrow G_kH _k^\top$,\; $Y_k \leftarrow U_kV_k^\top$,\; $W_k \leftarrow X_k\Had Y_k$
  \STATE $E_k \leftarrow \nabla_W f(W_k)$ \hfill\textit{(loss-specific)}
  \STATE Euclidean gradients via~\eqref{eq:egrad}
  \STATE Form S-blocks~\eqref{eq:sblocks_GH}--\eqref{eq:sblocks_UV}
  \STATE Riemannian gradient via~\eqref{eq:rgrad_G}--\eqref{eq:rgrad_H}
  \STATE $d_k \leftarrow -\mathrm{grad}\,f(x_k)$
  \STATE Cauchy step $\alpha_k^0$ via~\eqref{eq:cauchy}; Armijo back-tracking
  \STATE $x_{k+1} \leftarrow R_{x_k}(\alpha_k\,d_k)$ via~\eqref{eq:retraction}
\ENDFOR
\end{algorithmic}
\end{algorithm}

Algorithm~\ref{alg:rgd} summarizes RGD. Computing the gradient norm required for convergence monitoring has low computational overheads (Remark~\ref{rem:cheap_norm}).

\subsubsection{Riemannian trust-region method (RTR)}
\label{sec:rtr}
The same geometry supports a second-order method. A Riemannian trust-region method (RTR) minimizes a quadratic model of $f$ inside a truncated conjugate-gradient (Steihaug--Toint) ball, and needs three ingredients beyond RGD: a horizontal projection $P_x^{\cH}$ onto $\cH_x$ (so the Hessian maps the horizontal space to itself), the ambient Levi--Civita connection $\nabla$ of the metric~\eqref{eq:metric} (with Christoffel term $\Gamma_x$), and the Riemannian Hessian. With $\eta=\mathrm{grad}\,f$, the Euclidean Hessian-vector product $H_E[\xi]=D(\nabla f)[\xi]$, and $(D_\xi g)$ the directional derivative of the weight blocks~\eqref{eq:sblocks_GH}--\eqref{eq:sblocks_UV}, the Hessian on the quotient is
\begin{equation}
  \mathrm{Hess}\,f(x)[\xi] = P_x^{\cH}\,g^{-1}\!\Big(H_E[\xi] + \tfrac12\big[(D_\eta g)\xi - (D_\xi g)\eta - c(\xi,\eta)\big]\Big),
  \label{eq:hess_main}
\end{equation}
where $c(\xi,\eta)=\nabla_x[g_x(\xi,\eta)]$ is the reaction covector. Only $H_E$ depends on the loss (through $\nabla^2_W f$). Appendix~\ref{app:secondorder} derives all three ingredients (Propositions~\ref{prop:hproj}--\ref{prop:hess}) and proves that~\eqref{eq:hess_main} is the Riemannian Hessian on the quotient; their per-iteration cost is analyzed in Section~\ref{sec:complexity}. At each step RTR minimizes the quadratic model $m_k(\xi)=g_x(\mathrm{grad}\,f(x_k),\xi)+\tfrac12\,g_x(\mathrm{Hess}\,f(x_k)[\xi],\xi)$ over horizontal $\xi$ in the trust ball, summarized in Algorithm~\ref{alg:rtr}. 

\begin{algorithm}[t]
\caption{RTR on the Hadamard low-rank manifold}
\label{alg:rtr}
\begin{algorithmic}[1]
\STATE \textbf{Input:} smooth loss $f$, ranks $r_1,r_2$, initial trust radius $\Delta_0$
\STATE Initialize factors $G_0,H_0,U_0,V_0$ (SVD-based)
\FOR{$k=0,1,2,\dots$ \textbf{until} $\|\mathrm{grad}\,f\|_{x_k} < \mathrm{tol}$}
  \STATE $x_k=(G_k,H_k,U_k,V_k)$
  \STATE Riemannian gradient $\mathrm{grad}\,f(x_k)$ via~\eqref{eq:rgrad_G}--\eqref{eq:rgrad_H}
  \STATE Truncated-CG solve for horizontal $\xi_k$ with $\|\xi_k\|_{x_k}\le\Delta_k$, using $\mathrm{Hess}\,f(x_k)[\cdot]$~\eqref{eq:hess_main} and $P_{x_k}^{\cH}$ (Prop.~\ref{prop:hproj})
  \STATE Reduction ratio $\rho_k \leftarrow \big(f(x_k)-f(R_{x_k}(\xi_k))\big)/\big({-}m_k(\xi_k)\big)$
  \STATE Accept $x_{k+1}\leftarrow R_{x_k}(\xi_k)$ if $\rho_k$ sufficient, else $x_{k+1}\leftarrow x_k$; update $\Delta_{k+1}$
\ENDFOR
\end{algorithmic}
\end{algorithm}

\subsection{Per-iteration complexity}
\label{sec:complexity}

The per-iteration cost of RGD separates into a \emph{manifold cost} (independent of $f$) and a \emph{loss cost} (dependent on $f$). Below, we give details of the cost computations for the masked reconstruction problem cost $f(W) = \frac{1}{2}\|\Omega\Had(A-W)\|_F^2$ for an observation mask $\Omega$. Let $r = \max(r_1,r_2)$, $N = |\Omega|$ the number of observed entries, and $L$ the number of Armijo back-tracking steps (typically $L\leq 3$). Table~\ref{tab:complexity} summarizes the cost breakdown. We first account for the first-order (RGD) primitives, then for the additional ingredients of the trust-region method (RTR).


\noindent\textbf{S-block formation} ($O((m{+}n)r_1^2 r_2^2)$).
A naive computation of $S_i^G = H^\top\diag(Y_{i,:}^2)H$ would form the full $m\times n$ matrix $Y^2$, costing $O(mnr^2)$. We avoid this. Expanding $Y_{ij}^2 = (\sum_b U_{ib}V_{jb})^2 = \sum_{b,c} U_{ib}U_{ic}V_{jb}V_{jc}$ gives a tensor decomposition:
\begin{equation}
  S_i^G = \sum_{b,c=1}^{r_2} U_{ib}\,U_{ic}\;T_{bc}^G,
  \quad T_{bc}^G = H^\top\diag(V_{:,b}\Had V_{:,c})\,H.
  \label{eq:tensor_sblock}
\end{equation}
The $r_2^2$ tensor slices $T_{bc}^G \in \R^{r_1\times r_1}$ are precomputed once from the $n$-dimensional factors at cost $O(nr_1^2 r_2^2)$.
Assembling all $m$ blocks $S_i^G$ then costs $O(mr_1^2 r_2^2)$: for each row~$i$, form the weighted sum of $r_2^2$ pre-stored $r_1{\times}r_1$ matrices using the $r_2^2$ weights $U_{ib}U_{ic}$.
Repeating for all four S-block types ($S^G$, $S^H$, $S^U$, $S^V$) gives a total S-block cost of $O((m{+}n)r_1^2 r_2^2)$.
No $m\times n$ matrix is ever formed. 

\noindent\textbf{Row-wise solves} ($O((m{+}n)r^3)$).
The $(m{+}n)$ independent $r\times r$ linear systems~\eqref{eq:rgrad_G}--\eqref{eq:rgrad_H} are solved via Cholesky factorization, at cost $O(r^3)$ each. 
The Cholesky factorization is well defined because each $S$-block is positive definite (Proposition~\ref{prop:metric}). For sparse or poorly scaled data in which a block may become ill-conditioned, a small ridge $\epsilon I$ can be added before factorization (metric regularization) at negligible cost. The random orthogonal initialization used in Section~\ref{sec:experiments} further avoids degenerate initial $S$-blocks.

\noindent\textbf{Euclidean gradients} ($O(Nr)$ with sparse $\Omega$).
For losses involving only observed entries, the data-term gradient $\nabla_G f_{\mathrm{data}} = \sum_{(i,j) \in \Omega} E_{ij} Y_{ij}\,H_{j,:}$ is accumulated via scatter-add at cost $O(Nr)$.
The regulariser contribution $\lambda\,\nabla_G (\tfrac{1}{2}\|W\|_F^2) = \lambda\,(X\Had Y^2)\,H$ reuses the already-computed S-blocks:
$[\lambda\,(X\Had Y^2)\,H]_{i,:} = \lambda\,G_{i,:}\,S_i^G$, which costs $O(mr^2)$.
Total Euclidean gradient cost: $O(Nr + (m{+}n)r^2)$. 

\noindent\textbf{Cauchy step and line search.}
The Cauchy denominator $\|D\Phi_x[d]\|_{\Omega}^2 = \sum_{(i,j)\in\Omega} (D\Phi_x[d])_{ij}^2$ (the squared Frobenius norm restricted to observed entries) and each Armijo trial evaluation of $f(R_x(\alpha d))$ require reconstructing $W$ only at $\Omega$, at cost $O(Nr)$ per evaluation.
The regulariser $\|W\|_F^2$ is computed from factor Gram matrices at $O((m{+}n)r^4)$.
With $L$ Armijo steps, the line-search cost is $O(LNr + L(m{+}n)r^4)$. 

\noindent\textbf{Gradient norm} ($O((m{+}n)r)$).
By Remark~\ref{rem:cheap_norm}, the squared norm $\|\mathrm{grad}\,f\|^2 = \sum_F \tr(\nabla_F f^\top\,\mathrm{grad}_F f)$ reduces to a Euclidean dot product at cost $O((m{+}n)r)$. 

\noindent\textbf{Second-order ingredients (RTR)} ($N$-independent).
The trust-region method reuses the first-order primitives above and adds three costs, none of which touch the data term. The horizontal projection of Proposition~\ref{prop:hproj} is a matrix-free conjugate-gradient solve whose operator applies at $O((m{+}n)r^2)$ and forms no $m\times n$ matrix. The ambient Levi--Civita connection contributes Christoffel contractions of the metric-derivative blocks~\eqref{eq:sblocks_GH}--\eqref{eq:sblocks_UV}, again $O((m{+}n)r^2)$ through matrix products with no dense tensor stored. One Hessian-vector product~\eqref{eq:hess_main} then combines the batched $r\times r$ solves of the gradient, these metric-derivative contractions, and one horizontal projection.

\begin{table}[t]
\centering
\caption{Per-iteration cost breakdown.  $r = \max(r_1,r_2)$, $N = |\Omega|$, $L$ denotes the number of Armijo steps (usually ${\leq}3$), and LS in the last row denotes the least squares loss function.
}
\label{tab:complexity}
\small
\begin{tabular}{@{}lccc@{}}
\toprule
Operation & RGD & BCD & AGD \\
\midrule
S-blocks / normal eq. & $(m{+}n)r_1^2 r_2^2$ & $O(Nr^2)$ & --- \\
Row-wise $r{\times}r$ solves & $(m{+}n)r^3$ & $(m{+}n)r^3$ & --- \\
Eucl.\ gradients & $O(Nr)$ & --- & $O(Nr)$ \\
Cauchy + Armijo & $LNr {+} L(m{+}n)r^4$ & --- & --- \\
Gradient norm & $(m{+}n)r$ & --- & --- \\
\midrule
\textbf{Total} & $O(Nr {+} (m{+}n)r^4)$ & $O(Nr^2)$ & $O(Nr)$ \\
\midrule
Loss support & any $f$ & LS only & any $f$ \\
\bottomrule
\end{tabular}
\end{table}

\noindent\textbf{Comparison with BCD and AGD.}
BCD's total per-iteration cost is $O(Nr^2 + (m{+}n)r^3)$: each of four factor sweeps assembles masked normal equations at $O(Nr^2)$ and solves $(m{+}n)$ systems at $O(r^3)$ each (see Section~\ref{sec:related} for the sub-problem structure). BCD avoids line search (each sub-problem is solved exactly) and can be accelerated with heavy-ball momentum, but requires a least squares loss. 
RGD's dominant cost is $O(Nr + (m{+}n)r^4)$: an extra factor of $r$ in the $(m{+}n)$ term but a factor of $r$ less in the $N$ term compared to BCD, with the advantage of supporting any smooth loss.

AGD uses plain Euclidean gradient steps at $O(Nr)$ per iteration, the cheapest option but without curvature information, as reflected in the experimental results (Section~\ref{sec:experiments}).

\section{Numerical Experiments}
\label{sec:experiments}

We evaluate our first- and second-order methods in four experiments:
\begin{enumerate}[nosep,label=(\Alph*)]
\item \textbf{Full matrix approximation} (least squares, fully observed): synthetic robustness stress tests on $2{,}000{\times}1{,}000$ targets verifying scale invariance, conditioning adaptation, and noise robustness, plus real-world datasets.
\item \textbf{MovieLens-1M matrix completion} ($6{,}040{\times}3{,}706$, $4.5\%$ density), testing genuine Hadamard-rank configurations at fixed parameter budget $r_1{+}r_2{=}6$ (so that all configurations use the same $(m{+}n)(r_1{+}r_2)$ parameters).
\item \textbf{Logistic link prediction} (non-quadratic loss): ``like'' prediction on MovieLens-1M, where BCD is inapplicable and held-out AUC-ROC measures quality (Section~\ref{sec:exp_logistic}).
\item \textbf{Riemannian trust-region method (RTR)}: we compare against the concurrent second-order method of \citet{gillis2026manifold} on its own least-squares benchmarks (Section~\ref{sec:exp_secondorder}).
\end{enumerate}


\noindent\textbf{Baselines.}
We compare against BCD~\cite{wertz2024bcd} (heavy-ball momentum $\beta{=}0.8$) and AGD~\cite{ciaperoni2024hadamard} (normalized Euclidean gradient steps, wall-clock time matched to RGD).

\noindent\textbf{Common setup.}
All solvers share the same SVD initialization~\cite{wertz2024bcd}: $(G_0,H_0)$ from the leading singular vectors of $\sqrt{|\Omega\Had A|}$, $(U_0,V_0)$ from those of $\sqrt{|\Omega\Had A|}\cdot\mathrm{sign}(\Omega\Had A)$, followed by a global magnitude scaling. For RGD, the factors are additionally rotated by random orthogonal matrices to avoid degenerate initial S-blocks.
Regularization $\lambda$ is grid-searched over $\{10^{-6},\ldots,10^{-2}\}$ via a short RGD run on seed~0, then fixed for all solvers and seeds to ensure that differences reflect optimization quality rather than hyperparameter tuning.
All timings are wall-clock, and all solvers use the same NumPy/SciPy/Pymanopt~\cite{pymanopt} backend on the same machine, so their relative timings are directly comparable (hardware details in Appendix~\ref{app:expdetails}). RGD and BCD run for 300~iterations and AGD runs for as many iterations as fit within RGD's measured wall-clock time.

\subsection{Full matrix approximation}
\label{sec:exp_full}

\subsubsection{Real datasets}

We solve $\min \frac{1}{2}\|A - W\|_F^2 + \frac{\lambda}{2}\|W\|_F^2$ with $\lambda = 10^{-4}$ on three fully observed real-world matrices at symmetric ranks $r_1 = r_2 = r$:
Camera ($256{\times}256$, grayscale image),
Football ($115{\times}115$, network adjacency~\cite{girvan2002community}), and
Les Mis\'{e}rables ($77{\times}77$, co-appearance network~\cite{knuth1993stanford}).
Performance is measured by relative error $\|W-A\|_F / \|A\|_F$.

\begin{figure*}[t]
\centering
\includegraphics[width=0.95\textwidth]{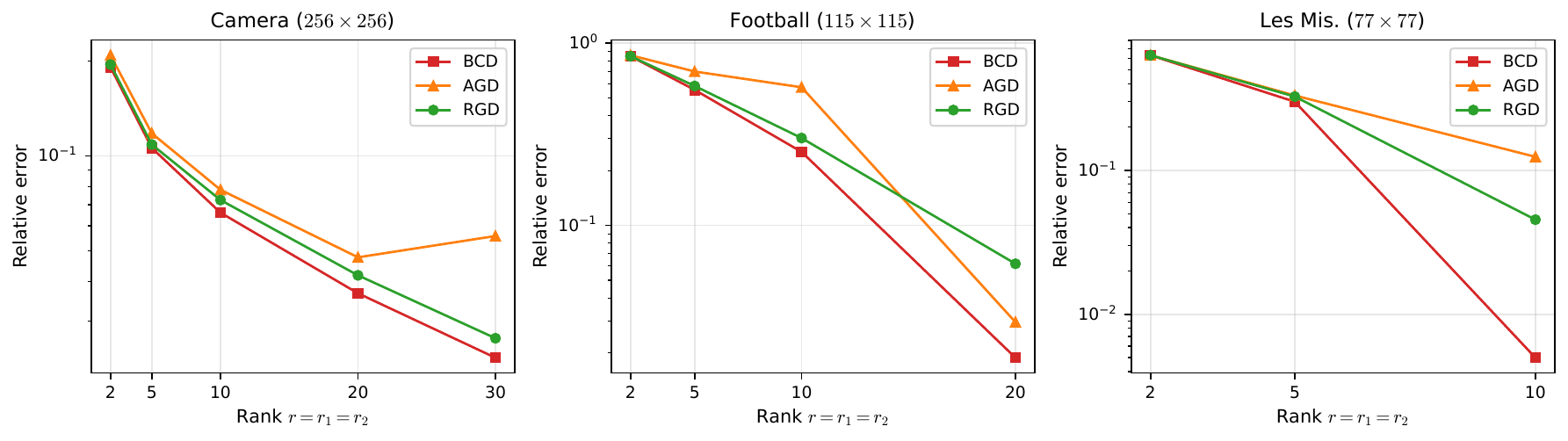}
\caption{Full matrix approximation: relative error vs.\ rank $r{=}r_1{=}r_2$ (log scale) on three real datasets.  All three solvers perform comparably, with BCD having a slight edge due to its exact sub-problem solves.}
\label{fig:full_approx}
\end{figure*}

Figure~\ref{fig:full_approx} shows that RGD tracks BCD closely across all datasets and ranks, confirming that the S-block metric provides BCD-equivalent preconditioning for quadratic losses, while AGD falls behind as rank increases (on Football at $r{=}10$, AGD reaches $0.57$ against $0.30$ for RGD and $0.25$ for BCD). BCD's exact sub-problem solves give it a small edge at higher ranks on the network data (e.g., Les~Mis.\ $r{=}10$: $0.005$ against $0.045$), as expected. RGD's advantage is its generalization to the non-quadratic losses where BCD is inapplicable.

\subsubsection{Synthetic robustness stress tests}
\label{sec:exp_robustness}

Experiment~A evaluated approximation quality on fixed datasets.
We now test solver robustness on synthetic $2{,}000{\times}1{,}000$ ground-truth targets of the form $A^\star = (G^\star H^{\star\top})\Had(U^\star V^{\star\top})$, where exact recovery is possible.
Performance is measured by the recovery error $\|W - A^\star\|_F / \|A^\star\|_F$ after 300 iterations (mean over 3~seeds).
We test robustness along three axes (factor conditioning, scale separation, and noise), varying one at a time while fixing the other two at their default values:

\begin{enumerate}[nosep]
\item \textbf{Factor conditioning} ($\kappa_{\mathrm{cond}} \in \{1,10,100,1{,}000\}$, with $\kappa_{\mathrm{scale}}{=}1$, $\sigma{=}0$):
the ground-truth factors $G^\star, H^\star, U^\star, V^\star$ are constructed with geometrically spaced singular values from $1$ to $1/\kappa_{\mathrm{cond}}$.
\item \textbf{Scale separation} ($\kappa_{\mathrm{scale}} \in \{1,10,100,1{,}000\}$, with $\kappa_{\mathrm{cond}}{=}10$, $\sigma{=}0$):
the SVD-based initialization is perturbed by diagonal scalings
$(G_0,H_0,U_0,V_0) \mapsto (D_r G_0, D_c H_0, D_r^{-1} U_0, D_c^{-1} V_0)$
with $D_r = \diag(\kappa_{\mathrm{scale}}^{(i{-}1)/(m{-}1)})$ and $D_c$ defined analogously.
This preserves the initial $W_0 = X_0 \Had Y_0$ but produces an adversarial factor representation.
Because the Riemannian metric is invariant under such scalings (Proposition~\ref{prop:gauge}(b)), RGD should produce identical iterates regardless of $\kappa_{\mathrm{scale}}$.
\item \textbf{Noise} ($\sigma \in \{0,0.01,0.1,0.5\}$, with $\kappa_{\mathrm{cond}}{=}10$, $\kappa_{\mathrm{scale}}{=}1$):
the target is corrupted as $A_{\mathrm{noisy}} = A^\star + \sigma\,(\|A^\star\|_F / \sqrt{mn})\,\mathcal{N}(0,1)$. Solvers are initialized from the noisy target and recovery is measured against the clean $A^\star$.
\end{enumerate}
We test three rank pairs at fixed budget $r_1{+}r_2{=}6$: $(1,5)$, $(2,4)$, $(3,3)$, each with 3~seeds.
All solvers are unregularized ($\lambda{=}0$) and share the same SVD-based initialization from the (possibly noisy) target.

\subsubsection{Conditioning sweep}

\begin{table}[t]
\centering
\caption{Conditioning sweep ($2{,}000{\times}1{,}000$, $\kappa_{\mathrm{scale}}{=}1$, $\sigma{=}0$): recovery error vs.\ factor condition number.
Mean over 3 seeds.  Best per row in bold.}
\label{tab:conditioning}
\small
\begin{tabular}{@{}c c c c c@{}}
\toprule
$(r_1,r_2)$ & $\kappa_{\mathrm{cond}}$ & RGD & BCD & AGD \\
\midrule
\multirow{4}{*}{$(1,5)$}
  & 1    & \textbf{0.0} & \textbf{0.0} & 9.5e-3 \\
  & 10   & \textbf{0.0} & \textbf{0.0} & 9.5e-2 \\
  & 100  & \textbf{0.0} & \textbf{0.0} & 4.3e-2 \\
  & 1000 & \textbf{0.0} & \textbf{0.0} & 1.4e-2 \\
\midrule
\multirow{4}{*}{$(2,4)$}
  & 1    & 3.5e-1 & \textbf{3.4e-1} & 6.3e-1 \\
  & 10   & \textbf{1.8e-3} & 2.8e-3 & 2.3e-2 \\
  & 100  & 3.9e-4 & \textbf{3.1e-4} & 6.4e-3 \\
  & 1000 & \textbf{6.4e-5} & 9.6e-5 & 1.3e-3 \\
\midrule
\multirow{4}{*}{$(3,3)$}
  & 1    & 2.5e-1 & \textbf{1.6e-1} & 7.6e-1 \\
  & 10   & \textbf{1.1e-2} & 1.4e-2 & 3.2e-2 \\
  & 100  & 1.8e-4 & \textbf{1.5e-4} & 3.0e-3 \\
  & 1000 & 2.2e-6 & \textbf{1.5e-6} & 3.3e-4 \\
\bottomrule
\end{tabular}
\end{table}

Table~\ref{tab:conditioning} shows that higher factor conditioning ($\kappa_{\mathrm{cond}}$) consistently improves all solvers.
At $\kappa_{\mathrm{cond}}{=}1$ (isotropic factors), the landscape has many symmetrically equivalent local minima, and all solvers struggle at genuinely Hadamard ranks.
RGD and BCD are comparable across ranks, while AGD errors are consistently one to two orders of magnitude larger, reflecting the absence of curvature information.
At the near-vanilla rank $(1,5)$, both RGD and BCD achieve machine-precision recovery at all conditioning levels (not shown for brevity).

\subsubsection{Scale invariance}

\begin{figure*}[t]
\centering
\includegraphics[width=0.95\textwidth]{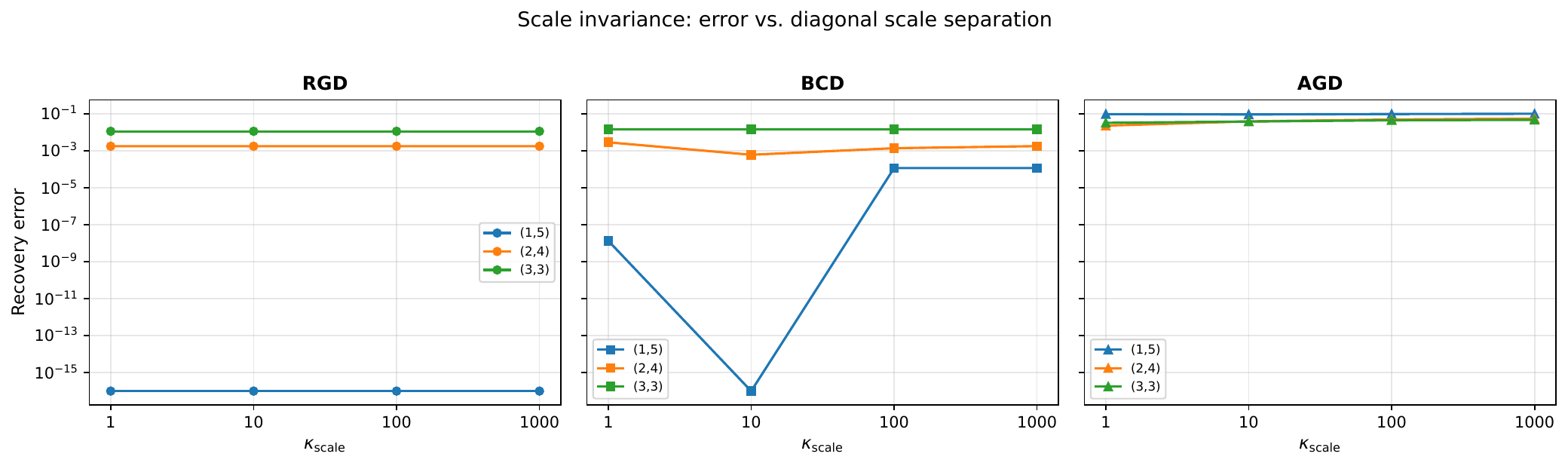}
\caption{Scale invariance: recovery error vs.\ diagonal scale separation $\kappa_{\mathrm{scale}}$ at three rank pairs ($\kappa_{\mathrm{cond}}{=}10$, $\sigma{=}0$, mean over 3~seeds).
RGD's error is \emph{exactly constant} across all $\kappa_{\mathrm{scale}}$ values for every rank (flat lines), confirming the metric invariance of Proposition~\ref{prop:gauge}(b).
BCD is nearly invariant at $(3,3)$ but shows mild sensitivity at $(2,4)$.
AGD degrades mildly ($1.1$--$2.4{\times}$ ratio).}
\label{fig:scale_invariance}
\end{figure*}

Figure~\ref{fig:scale_invariance} presents the key robustness result.
By Proposition~\ref{prop:gauge}(b), the Riemannian metric is invariant under cross-Hadamard scaling $(G,H,U,V) \mapsto (D_r G, D_c H, D_r^{-1}U, D_c^{-1}V)$.
To test this empirically, we perturb the initial factors by log-spaced diagonal matrices with spread $\kappa_{\mathrm{scale}}$ while preserving the initial product $W_0$.
As predicted, RGD produces \emph{identical} final errors across all $\kappa_{\mathrm{scale}} \in \{1, 10, 100, 1{,}000\}$ for every rank (max/min ratio $= 1.0000$ to machine precision, flat lines in the left panel).
This is a direct empirical verification of the metric invariance.

BCD is also approximately invariant at most ranks (its normal-equation matrices have the same structure as RGD's S-blocks), but shows mild sensitivity at $(2,4)$.
AGD's error varies by a factor of $1.1$--$2.4{\times}$ across scale separations: while the final $W$ is invariant under the scaling, the Euclidean gradient steps used by AGD are not, so the optimization trajectory changes.

\subsubsection{Noise robustness}

\begin{table}[t]
\centering
\caption{Noise sweep ($2{,}000{\times}1{,}000$, $\kappa_{\mathrm{cond}}{=}10$, $\kappa_{\mathrm{scale}}{=}1$): recovery error vs.\ noise level $\sigma$.
Mean over 3 seeds.  Best per row in bold.}
\label{tab:noise}
\small
\begin{tabular}{@{}c c c c c@{}}
\toprule
$(r_1,r_2)$ & $\sigma$ & RGD & BCD & AGD \\
\midrule
\multirow{4}{*}{$(1,5)$}
  & 0    & \textbf{0.0} & \textbf{0.0} & 9.5e-2 \\
  & 0.01 & \textbf{8.6e-4} & \textbf{8.6e-4} & 4.8e-2 \\
  & 0.1  & \textbf{8.7e-3} & \textbf{8.7e-3} & 3.6e-2 \\
  & 0.5  & \textbf{4.9e-2} & \textbf{4.9e-2} & 7.3e-2 \\
\midrule
\multirow{4}{*}{$(2,4)$}
  & 0    & \textbf{1.8e-3} & 2.8e-3 & 2.3e-2 \\
  & 0.01 & 2.3e-3 & \textbf{1.1e-3} & 4.3e-2 \\
  & 0.1  & \textbf{9.0e-3} & \textbf{9.0e-3} & 2.7e-2 \\
  & 0.5  & \textbf{5.5e-2} & 5.6e-2 & 4.9e-2 \\
\midrule
\multirow{4}{*}{$(3,3)$}
  & 0    & \textbf{1.1e-2} & 1.4e-2 & 3.2e-2 \\
  & 0.01 & \textbf{1.1e-2} & 1.5e-2 & 4.3e-2 \\
  & 0.1  & \textbf{1.4e-2} & 1.5e-2 & 4.6e-2 \\
  & 0.5  & \textbf{5.6e-2} & \textbf{5.6e-2} & 7.5e-2 \\
\bottomrule
\end{tabular}
\end{table}

Table~\ref{tab:noise} shows that RGD and BCD both converge to the noise floor: at $\sigma{=}0.01$ the recovery error matches $\sigma\sqrt{\dim(\cM)/(mn)} \approx 8.6{\times}10^{-4}$ for $(1,5)$, confirming that neither solver is impeded by the noisy initialization.
At high noise ($\sigma{=}0.5$), all solvers converge to similar error floors ($0.05$--$0.09$), since noise dominates the model mismatch.
AGD consistently trails at low noise levels (e.g., $4.3{\times}10^{-2}$ vs.\ $1.1{\times}10^{-2}$ for RGD at $(3,3)$, $\sigma{=}0.01$), again reflecting its lack of curvature adaptation.

\subsubsection{Summary}

\begin{figure*}[t]
\centering
\includegraphics[width=0.90\textwidth]{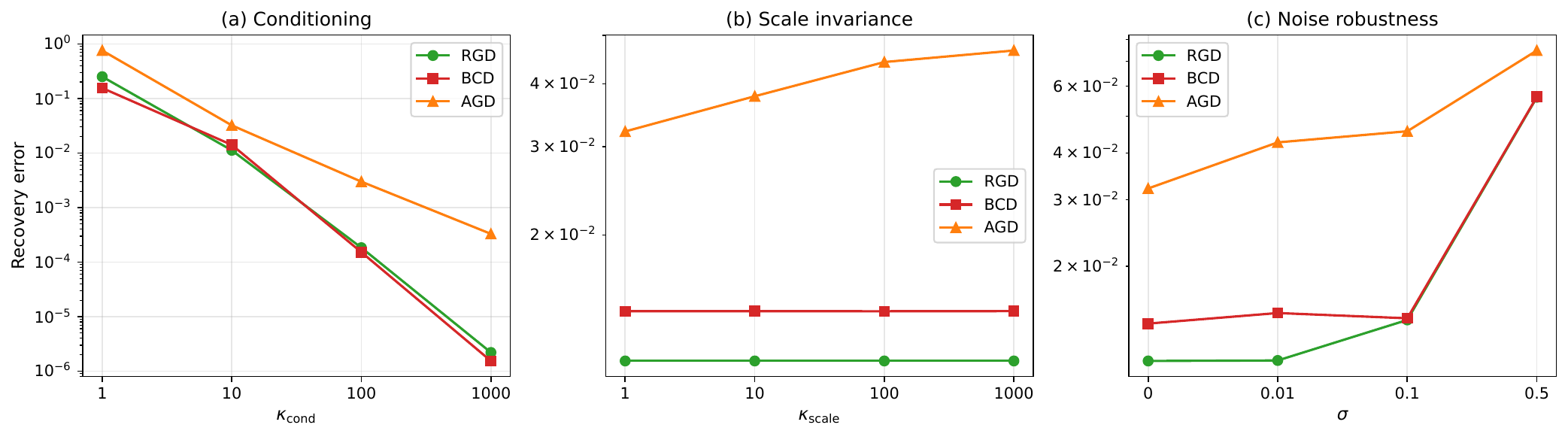}
\caption{Robustness at rank $(3,3)$ on $2{,}000{\times}1{,}000$ synthetic targets (mean over 3~seeds).
(a)~Conditioning: higher $\kappa_{\mathrm{cond}}$ helps all solvers; RGD and BCD are comparable.
(b)~Scale invariance: RGD is perfectly flat (Proposition~\ref{prop:gauge}); AGD degrades mildly.
(c)~Noise: RGD and BCD track the noise floor; AGD lags.
Errors at machine precision are clamped at $10^{-16}$ for the log scale.}
\label{fig:robustness_33}
\end{figure*}

Figure~\ref{fig:robustness_33} summarizes all three sweeps at rank $(3,3)$. They confirm that the S-block metric matches BCD's exact solves under conditioning, that RGD is invariant under cross-Hadamard scaling in both theory (Proposition~\ref{prop:gauge}) and practice, and that RGD tracks the noise floor while AGD lags.

\subsection{MovieLens-1M collaborative filtering}
\label{sec:exp_movielens}

To evaluate at a practical scale, we test on the \textbf{MovieLens-1M} benchmark~\cite{harper2015movielens}, which contains $6{,}040$ users, $3{,}706$ movies, and $1{,}000{,}000$ ratings (4.5\% density).
The objective is the regularized masked least-squares cost:
\begin{equation*}
  \min_{G,H,U,V}\; \tfrac{1}{2}\|\Omega\Had(A - W)\|_F^2 + \tfrac{\lambda}{2}\|W\|_F^2,
\end{equation*}
where $\|W\|_F^2$ can be computed cheaply (cost $O((m{+}n)r^4)$).
BCD~\cite{wertz2024bcd} is originally designed for the full decomposition, so we adapt it to completion by restricting its per-row normal equations to $\Omega$, while AGD~\cite{ciaperoni2024hadamard} applies directly as a smooth masked loss. We report held-out RMSE over rank pairs at fixed budget $r_1{+}r_2{=}6$ with varying interaction capacity $r_1 r_2\in\{5,8,9\}$, namely near-vanilla $(1,5)$, asymmetric $(2,4)$, and balanced $(3,3)$, at training fractions $80\%$ and $50\%$ over 3~seeds.

\begin{table}[t]
\centering
\caption{MovieLens-1M ($r_1{+}r_2{=}6$): held-out RMSE.
Mean over 3 seeds.  Best per row in bold.}
\label{tab:budget}
\small
\begin{tabular}{@{}c c c c c c@{}}
\toprule
$(r_1,r_2)$ & $r_1 r_2$ & $p_{\mathrm{tr}}$ & BCD & AGD & RGD \\
\midrule
\multirow{2}{*}{$(1,5)$} & \multirow{2}{*}{5}
  & 0.8 & \BMorig{\textbf{0.88}} & 2.96 & \textbf{0.88} \\
  & & 0.5 & \BMorig{\textbf{0.91}} & 5.11 & \textbf{0.91} \\
\midrule
\multirow{2}{*}{$(2,4)$} & \multirow{2}{*}{8}
  & 0.8 & 0.89 & 2.95 & \textbf{0.88} \\
  & & 0.5 & 0.94 & 4.56 & \textbf{0.91} \\
\midrule
\multirow{2}{*}{$(3,3)$} & \multirow{2}{*}{9}
  & 0.8 & 0.90 & 2.98 & \textbf{0.87} \\
  & & 0.5 & 0.95 & 3.51 & \textbf{0.91} \\
\bottomrule
\end{tabular}
\end{table}

Table~\ref{tab:budget} shows that RGD attains the best RMSE at every configuration. At the fixed budget $r_1{+}r_2{=}6$ the balanced split $(3,3)$ reaches $0.87$ against $(1,5)$'s $0.88$, so the gain comes from the interaction capacity $r_1 r_2$ rather than the parameter count. BCD matches RGD at the near-vanilla $(1,5)$ but trails at genuinely Hadamard ranks ($0.90$ against $0.87$ at $(3,3)$, $p_{\mathrm{tr}}{=}0.8$), as the convergence panel in Figure~\ref{fig:convergence} confirms, whereas AGD never reaches competitive RMSE (2.95--5.11).

\begin{figure}[t]
\centering
\includegraphics[width=\columnwidth]{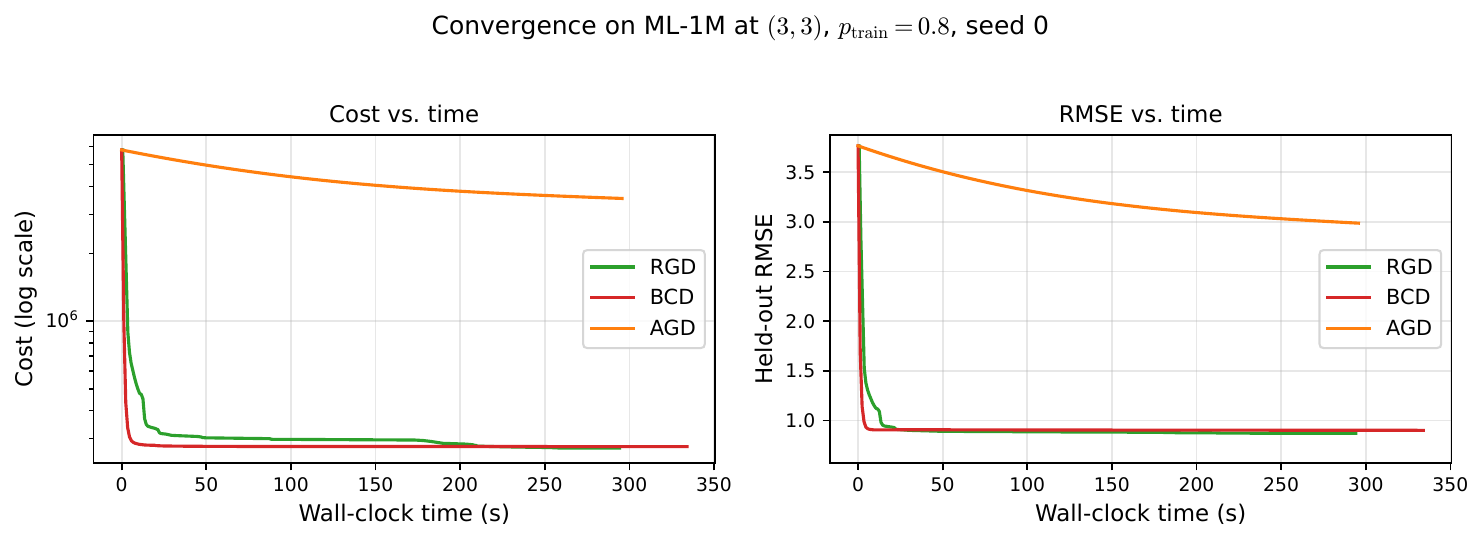}
\caption{Convergence on ML-1M at $(3,3)$, $p_{\mathrm{train}}{=}0.8$, seed~0.  Left: cost vs.\ wall-clock time (log scale).  Right: held-out RMSE vs.\ time.  RGD converges fastest and reaches the lowest RMSE.}
\label{fig:convergence}
\end{figure}


\subsection{Beyond least squares: logistic link prediction}
\label{sec:exp_logistic}

Adapting our framework to a non-quadratic loss only requires changing the gradient $E = \nabla_W f$ and leaves the metric, the weight blocks, and the retraction untouched. This is where the framework separates most clearly from BCD, which solves per-row least-squares normal equations in closed form. 

We model a binary matrix $A\in\{0,1\}^{m\times n}$ through the logistic link and minimize the regularized cross-entropy over an observed set $\Omega$,
\begin{equation}
  f(W)=\sum_{(i,j)\in\Omega}\big[\log(1+e^{W_{ij}})-A_{ij}W_{ij}\big]+\tfrac{\lambda}{2}\|W\|_F^2,
  \label{eq:logistic}
\end{equation}
with $W=(GH^\top)\Had(UV^\top)$ and $\sigma(z)=1/(1+e^{-z})$. The $W$-space gradient of~\eqref{eq:logistic} is $E_{ij}=\sigma(W_{ij})-A_{ij}$ on $\Omega$ and zero elsewhere, plus the regularizer term $\lambda W$, and it enters the factor gradients through~\eqref{eq:egrad} with no other change.

We first evaluate on real data. On MovieLens-1M we binarize the ratings into a ``like'' indicator ($A_{ij}{=}1$ iff the rating is at least $4$, a positive fraction of $0.58$) and predict held-out likes at the same fixed parameter budget $r_1{+}r_2{=}6$ used for completion. We train on a fraction $p_{\mathrm{train}}$ of the observed ratings and report held-out AUC-ROC on the remainder. Both solvers share the SVD-based initialization of Section~\ref{sec:experiments}, and AGD~\cite{ciaperoni2024hadamard} receives the same wall-clock budget as RGD's $300$ iterations. Table~\ref{tab:logistic_ml1m} reports the mean and standard deviation over three seeds.

RGD improves held-out AUC over AGD by $5$ to $7$ points in every configuration, with standard deviation at most $0.002$, and the margin widens as the ranks become balanced ($(1,5)\!\to\!(3,3)$) and as the training fraction grows, exactly the regime in which the per-row Hadamard preconditioning is most valuable. BCD does not apply, because the logistic loss admits no closed-form per-row least-squares solution. Figure~\ref{fig:logistic_conv} traces the corresponding convergence on the $(3,3)$, $p_{\mathrm{train}}{=}0.8$ configuration: within the shared wall-clock budget RGD drives the objective lower and climbs to a higher held-out AUC, whereas the unpreconditioned AGD plateaus early.

\begin{table}[t]
\centering
\small
\caption{Logistic ``like'' prediction on MovieLens-1M ($6{,}040\times3{,}706$, ratings binarized at ${\geq}4$, positive fraction $0.58$), at fixed budget $r_1{+}r_2{=}6$ and training fraction $p_{\mathrm{train}}$. Held-out AUC-ROC, mean $\pm$ standard deviation over three seeds, higher is better. Best per row in bold. BCD~\cite{wertz2024bcd} is inapplicable to the non-quadratic loss.}
\label{tab:logistic_ml1m}
\setlength{\tabcolsep}{4pt}
\begin{tabular}{@{}llccc@{}}
\toprule
$(r_1,r_2)$ & $p_{\mathrm{train}}$ & RGD (ours) & AGD & BCD \\
\midrule
\multirow{2}{*}{$(1,5)$} & $0.5$ & $\mathbf{0.792}{\scriptstyle\,\pm0.001}$ & $0.745{\scriptstyle\,\pm0.002}$ & n/a \\
                         & $0.8$ & $\mathbf{0.804}{\scriptstyle\,\pm0.001}$ & $0.748{\scriptstyle\,\pm0.001}$ & n/a \\
\multirow{2}{*}{$(2,4)$} & $0.5$ & $\mathbf{0.788}{\scriptstyle\,\pm0.000}$ & $0.732{\scriptstyle\,\pm0.001}$ & n/a \\
                         & $0.8$ & $\mathbf{0.800}{\scriptstyle\,\pm0.000}$ & $0.733{\scriptstyle\,\pm0.001}$ & n/a \\
\multirow{2}{*}{$(3,3)$} & $0.5$ & $\mathbf{0.783}{\scriptstyle\,\pm0.002}$ & $0.726{\scriptstyle\,\pm0.001}$ & n/a \\
                         & $0.8$ & $\mathbf{0.798}{\scriptstyle\,\pm0.001}$ & $0.725{\scriptstyle\,\pm0.001}$ & n/a \\
\bottomrule
\end{tabular}
\end{table}

\begin{figure}[t]
\centering
\includegraphics[width=\columnwidth]{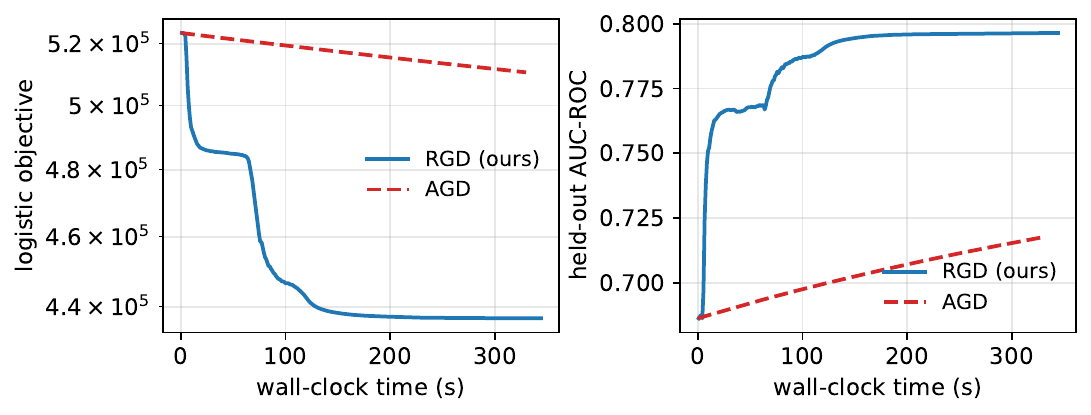}
\caption{Convergence on MovieLens-1M logistic ``like'' prediction at $(r_1,r_2){=}(3,3)$, $p_{\mathrm{train}}{=}0.8$, seed~$0$.  Left: logistic objective vs.\ wall-clock time (log scale).  Right: held-out AUC-ROC vs.\ time.  RGD (ours) and the time-matched AGD start from the same initialization; RGD converges faster and reaches a higher held-out AUC.}
\label{fig:logistic_conv}
\end{figure}

\subsection{Comparison with the concurrent second-order method}
\label{sec:exp_secondorder}

We compare our second-order method (RTR) against the second-order Manopt method of \citet{gillis2026manifold} on their own least-squares benchmarks: the Camera image (their Section~5.2) at $r\in\{8,16\}$, and $400\times400$ matrices with uniformly random entries (their Section~5.1) at $r=10$ over three seeds. Both methods start from the same face-splitting initialization and run for the same number of iterations. Table~\ref{tab:gillis_ls} reports the relative Frobenius error. Across every benchmark the two methods agree to within a few tenths of a percentage point. Figure~\ref{fig:rtr_gillis} shows that they follow essentially the same trajectory and reach the same plateau. Neither reaches the gradient tolerance from a cold start. The error decreases slowly while the gradient norm plateaus and the trust radius freezes, so the two methods meet at the plateau and the initialization matters more than the optimizer. This behavior is a property of the ill-conditioned Hadamard least-squares landscape and is shared by both methods.

\begin{table}[t]
\centering
\small
\caption{RTR (ours) against the second-order method of \citet{gillis2026manifold} on their own least-squares benchmarks, from a matched face-splitting initialization and equal per-benchmark iteration budgets. Relative Frobenius error in percent, mean over three seeds for the synthetic case.}
\label{tab:gillis_ls}
\begin{tabular}{@{}lcc@{}}
\toprule
Benchmark & RTR (ours) & Gillis et al. \\
\midrule
Camera, $r=8$   & $7.28$ & $7.36$ \\
Camera, $r=16$  & $4.57$ & $4.50$ \\
Synthetic $400\times400$, $r=10$ & $44.94\pm0.02$ & $45.21\pm0.07$ \\
\bottomrule
\end{tabular}
\end{table}

\begin{figure*}[t]
\centering
\includegraphics[width=0.82\textwidth]{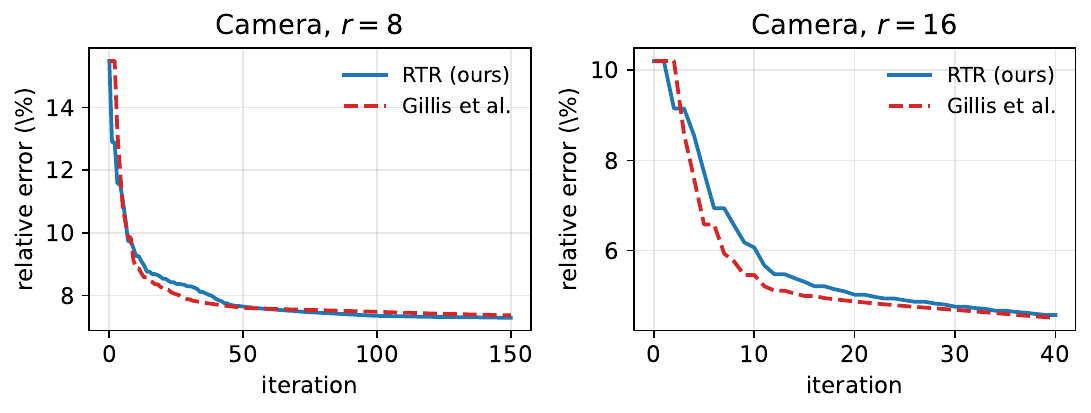}
\caption{RTR (ours) and the second-order method of \citet{gillis2026manifold} on the Camera image at $r=8$ and $r=16$, from the same face-splitting initialization. The two methods track each other throughout and reach the same plateau.}
\label{fig:rtr_gillis}
\end{figure*}


\section{Conclusion}
\label{sec:conclusion}

We cast the learning of Hadamard products of low-rank matrices as optimization on a Riemannian quotient manifold that removes both the general-linear reparameterization and the cross-Hadamard row and column scaling symmetries. On this manifold we proposed a block-diagonal metric taken from the block-diagonal part of the pullback of the Frobenius inner product, and showed that its weight blocks coincide with the per-row normal-equation matrices of BCD. Because the metric depends only on the factors and not on the loss, the resulting Riemannian gradient descent reuses BCD's per-row preconditioning for any smooth objective, at an $O(Nr + (m{+}n)r^4)$ per-iteration cost and with a tuning-free Gauss--Newton Cauchy step. Empirically, RGD tracks BCD on least-squares approximation across three real datasets, attains the lowest (or tied-lowest) RMSE at every tested Hadamard rank on MovieLens-1M, and exhibits the exact invariance to cross-Hadamard scaling predicted by the theory.

The same construction reaches well beyond least squares. It optimizes a non-quadratic logistic link-prediction objective on which BCD is inapplicable (Section~\ref{sec:exp_logistic}), and the quotient geometry is rich enough to admit a closed-form Levi--Civita connection and the Riemannian Hessian, yielding a second-order trust-region method that matches the concurrent method of \citet{gillis2026manifold} on its own benchmarks (Sections~\ref{sec:rtr} and~\ref{sec:exp_secondorder}). It also extends to products of $k$ low-rank factors, where each S-block keeps its $r\times r$ form with the complementary factor replaced by the elementwise product of the remaining $k{-}1$ factors, and to other Riemannian solvers such as conjugate gradients or trust regions~\cite{absilbook, ciaperoni2024hadamard, wertz2024bcd}. We view the identification of the cross-Hadamard scaling symmetry, together with a metric that is invariant under it, as the step that turns BCD's loss-specific preconditioning into a loss-agnostic Riemannian framework.

\bibliographystyle{IEEEtranN}
\bibliography{references}

\clearpage
\useRomanappendicesfalse
\appendices
\section*{Appendix Contents}
{\small
\noindent\ref{app:proofs}\quad Proofs\dotfill\pageref{app:proofs}\par
\noindent\ref{app:related}\quad Related Work Details\dotfill\pageref{app:related}\par
\noindent\ref{app:expdetails}\quad Experimental Details\dotfill\pageref{app:expdetails}\par
\noindent\ref{app:secondorder}\quad Riemannian Trust-Region Method: Derivations\dotfill\pageref{app:secondorder}\par
\noindent\ref{app:gillis}\quad The Concurrent Method of \citet{gillis2026manifold}: Positioning and Cost\dotfill\pageref{app:gillis}\par
}
\bigskip

\section{Proofs}
\label{app:proofs}
This appendix collects the proofs deferred from the main text. Throughout we work on the regular subset $\cM_*$ of Definition~\ref{def:metric} (full-column-rank factors with $X=GH^\top$ and $Y=UV^\top$ having no zero entries), on which every S-block~\eqref{eq:sblocks_GH}--\eqref{eq:sblocks_UV} is positive definite.

\begin{proof}[Proof of Proposition~\ref{prop:metric} (positive definiteness)]
Bilinearity and symmetry are immediate: each term is bilinear in $(\eta,\zeta)$ and each $S$-block is symmetric ($S = M^\top D M$ with $D$ diagonal).
For positive definiteness, suppose $\eta \neq 0$. Then at least one factor, say $\eta_G$, has a nonzero row $\eta_{G,i:}$.
Since $S_i^G = H^\top\diag(Y_{i,:}^2)H$ with $H$ full column rank and $Y_{i,:}^2$ entry-wise positive (on the open dense subset where $Y$ has no zero entries), $S_i^G$ is positive definite, so $\eta_{G,i:}S_i^G\eta_{G,i:}^\top > 0$.
All other terms are non-negative, giving $g_x(\eta,\eta) > 0$.
\end{proof}

\begin{proof}[Proof of Proposition~\ref{prop:gauge} ($\cG$-invariance)]
(a) Under $(\tilde G,\tilde H) = (GQ_1, HQ_1^{-\top})$, the weight block transforms as $S_i^{\tilde G} = Q_1^{-1}S_i^G Q_1^{-\top}$ and the tangent row as $\tilde\eta_{G,i:} = \eta_{G,i:}Q_1$.
The $G$-block becomes $\sum_i (\eta_{G,i:}Q_1)(Q_1^{-1}S_i^G Q_1^{-\top})(Q_1^\top\zeta_{G,i:}^\top) = \sum_i \eta_{G,i:}S_i^G\zeta_{G,i:}^\top$.
Analogous cancellations hold for $H$, $U$, $V$.
(b) Under the scaling, $S_i^{\tilde G} = d_{r,i}^{-2} S_i^G$ while $\tilde\eta_{G,i:} = d_{r,i}\,\eta_{G,i:}$.
The factors $d_{r,i}^2 \cdot d_{r,i}^{-2} = 1$ cancel, leaving the $G$-block invariant. The same holds for $H$, $U$, $V$.
\end{proof}

\section{Related Work Details}
\label{app:related}
This appendix expands the discussion of Section~\ref{sec:related}.

\paragraph{Fixed-rank quotient geometry.}
The set of $m\times n$ matrices of rank $r$ is a smooth manifold that admits several equivalent representations. Embedded formulations optimize directly on the rank-$r$ variety~\cite{vandereycken2013low}, whereas balanced factorizations $W = LR^\top$ introduce the general-linear reparameterization symmetry $(L,R)\mapsto(LQ, RQ^{-\top})$ and are naturally treated as a quotient $(\R_*^{m\times r}\times\R_*^{n\times r})/\mathrm{GL}(r)$~\cite{mishra2014r3mc,mishra2012riemannian}. In both cases the choice of Riemannian metric is the central design decision, because it fixes the search directions and the conditioning of the resulting optimization problem.

\paragraph{Riemannian preconditioning.}
Mishra and Sepulchre~\cite{mishra2016riemannian} formalize the view that a Riemannian metric on a quotient manifold acts as a preconditioner: metrics built from the block-diagonal part of the (Gauss--Newton) Hessian, or from normal-equation matrices, accelerate convergence relative to the Euclidean metric while keeping the per-iteration cost low. Our metric~\eqref{eq:metric} instantiates this principle for the Hadamard model: it retains the block-diagonal part of the pullback of the Frobenius inner product and discards the cross-factor terms, yielding independent $r\times r$ solves per row/column.

\paragraph{Hadamard low-rank models and solvers.}
The Hadamard product of low-rank factors models multiplicative structure that a single low-rank factor cannot capture at the same parameter budget, and arises in recommender systems~\cite{koren2009matrix}, link prediction with node-level heterogeneity~\cite{hoff2002latent}, count-data and single-cell models~\cite{townes2019glmpca}, and topic models and NMF~\cite{lee1999nmf}. Parameter-efficient adaptation~\cite{hu2022lora} is motivated by the same compactness. \citet{ciaperoni2024hadamard} introduced the Hadamard decomposition problem and an alternating gradient descent (AGD) solver applicable to any smooth loss but using uniform Euclidean steps. \citet{wertz2024bcd} proposed a block coordinate descent (BCD) that solves per-row least-squares normal equations in closed form and is accelerated with heavy-ball momentum~\cite{polyak1964heavy}, but is tied to the least-squares loss. Our approach recovers BCD's per-row preconditioning as a Riemannian metric that applies to any smooth loss.

\paragraph{Software.}
General-purpose Riemannian solvers and the manifolds used here are provided by Manopt~\cite{manopt}, Pymanopt~\cite{pymanopt}, and Manopt.jl~\cite{manoptjl}. We refer to~\cite{absilbook,boumal2023introduction} for the underlying theory of optimization on manifolds.

\section{Experimental Details}
\label{app:expdetails}
This appendix provides the dataset, solver, and protocol details used in Section~\ref{sec:experiments}.

\subsection{Datasets}
\begin{itemize}[nosep]
  \item \textbf{Camera} ($256{\times}256$): the standard grayscale ``cameraman'' image, resized to $256{\times}256$ with anti-aliasing.
  \item \textbf{Football} ($115{\times}115$): the college-football network adjacency matrix~\cite{girvan2002community}.
  \item \textbf{Les Mis\'{e}rables} ($77{\times}77$): the co-appearance network of characters~\cite{knuth1993stanford}.
  \item \textbf{Synthetic} ($2{,}000{\times}1{,}000$): ground-truth targets $A^\star = (G^\star H^{\star\top})\Had(U^\star V^{\star\top})$ with standard-normal factors, optionally reconditioned, rescaled, or corrupted by Gaussian noise for the robustness sweeps.
  \item \textbf{MovieLens-1M} ($6{,}040$ users $\times$ $3{,}706$ movies, $\sim\!10^6$ ratings, $4.5\%$ density)~\cite{harper2015movielens}, used for masked matrix completion.
\end{itemize}

\subsection{Solvers and Hyperparameters}
\begin{itemize}[nosep]
  \item \textbf{RGD} (ours): a Gauss--Newton Cauchy initial step~\eqref{eq:cauchy} followed by Armijo back-tracking with sufficient-decrease constant $c_1 = 10^{-4}$, contraction factor $\tau = 0.5$, and at most $25$ back-tracking steps. We run RGD for $300$ iterations or until the relative Riemannian gradient norm falls below $\mathrm{tol} = 10^{-8}$.
  \item \textbf{BCD}~\cite{wertz2024bcd}: exact per-row normal-equation solves with heavy-ball momentum ($\beta = 0.8$), run for $300$ iterations. Applicable to the least-squares loss only.
  \item \textbf{AGD}~\cite{ciaperoni2024hadamard}: normalized Euclidean gradient steps, run for as many iterations as fit within RGD's measured wall-clock time (fair time-matched comparison).
\end{itemize}

\subsection{Initialization}
All solvers share the SVD-based initialization of~\cite{wertz2024bcd}: $(G_0,H_0)$ are taken from the leading singular vectors of $\sqrt{|\Omega\Had A|}$ and $(U_0,V_0)$ from those of $\sqrt{|\Omega\Had A|}\cdot\mathrm{sign}(\Omega\Had A)$, followed by a global magnitude scaling. For RGD, the factors are additionally rotated by random orthogonal matrices to avoid degenerate initial S-blocks.

\subsection{Regularization and Protocol}
We grid-search the regularization weight $\lambda$ over $\{10^{-6},10^{-5},10^{-4},10^{-3},10^{-2}\}$ using a short RGD run on seed~$0$ and then fix it across all solvers and seeds. The full-matrix-approximation experiments use $\lambda = 10^{-4}$. We average every configuration over $3$ random seeds ($0,1,2$). All timings are wall-clock on a single workstation with $16$ logical CPU cores (AMD Ryzen), using NumPy/SciPy/Pymanopt~\cite{pymanopt} with default multi-threaded BLAS and no GPU. Because all solvers share this backend, their relative timings are directly comparable, whereas absolute times and any hardware-dependent crossovers (for instance from BLAS threading or cache effects) will vary across machines.

\subsection{Robustness Sweeps}
The synthetic stress tests vary one axis at a time (factor conditioning $\kappa_{\mathrm{cond}}$, diagonal scale separation $\kappa_{\mathrm{scale}}$, and additive noise $\sigma$) while holding the others at their defaults, over three rank pairs at the fixed budget $r_1{+}r_2 = 6$, namely $(1,5)$, $(2,4)$, and $(3,3)$, each over $3$ seeds. The precise constructions of the three axes are given in Section~\ref{sec:exp_robustness}.

\subsection{Metrics}
We report relative error $\|W-A\|_F/\|A\|_F$ for full-matrix approximation, held-out RMSE for MovieLens completion, and recovery error $\|W-A^\star\|_F/\|A^\star\|_F$ for the synthetic recovery tasks.

\section{Riemannian Trust-Region Method: Derivations}
\label{app:secondorder}
The main body builds RGD from three geometric objects, the metric~\eqref{eq:metric}, the Riemannian gradient~\eqref{eq:rgrad_G}--\eqref{eq:rgrad_H}, and the retraction~\eqref{eq:retraction}. A Riemannian trust-region method (RTR), which minimizes a quadratic model inside a truncated conjugate-gradient (Steihaug--Toint) subproblem (for details refer tto \citep{absilbook,boumal2023introduction}), needs three ingredients that RGD does not:
\begin{enumerate}[nosep,label=(\roman*)]
  \item the horizontal projection $P_x^{\cH}$ onto $\cH_x$, because the Hessian must map the horizontal space to itself and the inner solver operates within it,
  \item the ambient Levi--Civita connection $\nabla$ of the metric~\eqref{eq:metric}, and
  \item the Riemannian Hessian, assembled from the connection and the Euclidean Hessian of the loss.
\end{enumerate}
We derive each below and state its computation cost. All three depend only on the metric, and the Euclidean Hessian is the only cost-dependent quantity, so RTR applies to any smooth loss $f$.

\subsection{Ingredient 1: the horizontal projection}
We have the following expression for the horizontal projection operation.

\begin{proposition}[Horizontal projection]\label{prop:hproj}
Let $\cV_x$ be the vertical space of Section~\ref{sec:geometry}, and let $P_x^{\cH}=\mathrm{Id}-P_x^{\cV}$ where $P_x^{\cV}$ is the $g_x$-orthogonal projection onto $\cV_x$. For an ambient vector $Z$, the vertical part $v=P_x^{\cV}Z$ is the unique $v\in\cV_x$ satisfying $g_x(Z-v,w)=0$ for all $w\in\cV_x$. Parameterizing $v$ by $(\Lambda_1,\Lambda_2,\delta_r,\delta_c)$ as in $\cV_x$, the row and column scalings $\delta_r,\delta_c$ solve in closed form given $(\Lambda_1,\Lambda_2)$, leaving a coupled system in $(\Lambda_1,\Lambda_2)$ of size $r_1^2+r_2^2$ that is symmetric and positive semidefinite in $g_x$.
\end{proposition}
\begin{proof}
The $g_x$-orthogonal projection of $Z$ onto the subspace $\cV_x$ is the unique $v\in\cV_x$ that minimizes $\|Z-v\|_{g_x}$, characterized by the normal equations $g_x(Z-v,w)=0$ for all $w\in\cV_x$. Testing against the four generator families of $\cV_x$ produces four groups of equations in $(\Lambda_1,\Lambda_2,\delta_r,\delta_c)$. By the block-diagonal form of~\eqref{eq:metric}, row $i$ of $G$ and $U$ enters the diagonal generators only through the single scalar $\delta_{r,i}$, and column $j$ of $H$ and $V$ only through $\delta_{c,j}$. The equations indexed by $\delta_r,\delta_c$ are therefore scalar and, for fixed $(\Lambda_1,\Lambda_2)$, solve in closed form as affine functions of $(\Lambda_1,\Lambda_2)$. Substituting these back into the equations indexed by $(\Lambda_1,\Lambda_2)$ leaves a linear system whose operator is the Schur complement, with respect to the diagonal block, of the Gram matrix of the generators in $g_x$. A Gram matrix is symmetric positive semidefinite, and its Schur complement inherits both properties. The kernel is the one-dimensional overlap of the reparameterization and scaling generators, along which conjugate gradient terminates harmlessly.
\end{proof}

We solve the reduced $(\Lambda_1,\Lambda_2)$ system matrix-free with conjugate gradient. Each application of its operator costs $O((m{+}n)r^2)$ and forms no $m\times n$ matrix, so the projection scales linearly in the ambient dimensions. The one-dimensional gauge redundancy between the reparameterization and scaling actions makes the operator singular but positive semidefinite, which conjugate gradient handles without modification.

\subsection{Ingredient 2: the Levi--Civita connection}
The connection on the quotient is the horizontal projection of the connection on the total space, $P_x^{\cH}\big(\nabla_\xi\eta\big)$ for horizontal $\xi$ and $\eta$~\cite{absilbook,boumal2023introduction}. It therefore suffices to compute the ambient connection $\nabla$ and to apply Ingredient~1. On the total space of factors, a vector space, $\nabla$ splits into the ordinary directional derivative and a Christoffel correction, given in closed form by Proposition~\ref{prop:conn}.

\begin{proposition}[Ambient connection]\label{prop:conn}
The Levi--Civita connection of the metric~\eqref{eq:metric}, that is the unique torsion-free and metric-compatible connection, is
\begin{equation*}
  \nabla_\xi\eta = D\eta[\xi] + \Gamma_x(\xi,\eta),
\end{equation*}
where the Christoffel term $\Gamma_x$, nonzero because the weight blocks~\eqref{eq:sblocks_GH}--\eqref{eq:sblocks_UV} depend on $x$, is the unique tensor satisfying the Koszul identity
\begin{align}
  2\,g_x\!\big(\Gamma_x(\xi,\eta),\zeta\big) &= (D_\xi g)(\eta,\zeta) + (D_\eta g)(\xi,\zeta) \notag\\
  &\quad - (D_\zeta g)(\xi,\eta)
  \label{eq:app_koszul}
\end{align}
for all $\zeta$. Equivalently, because $g_x$ is nondegenerate,
\begin{equation}
  \Gamma_x(\xi,\eta) = \tfrac12\,g^{-1}\big[(D_\xi g)\eta + (D_\eta g)\xi - c(\xi,\eta)\big],
  \label{eq:app_gamma}
\end{equation}
where $(D_\xi g)\eta$ is the covector with $b$-block $\eta_b\,(DS_b[\xi])$, and the reaction covector $c(\xi,\eta)$ represents $\zeta\mapsto(D_\zeta g)(\xi,\eta)$ and equals $\nabla_x\big[g_x(\xi,\eta)\big]$. Here $(D_\xi g)(\eta,\zeta)=\sum_{b}\eta_b\,(DS_b[\xi])\,\zeta_b^\top$ over the four block families $b\in\{(G,i),(H,j),(U,i),(V,j)\}$, and, using $\Delta X,\Delta Y$ from Section~\ref{sec:cauchy},
\begin{align}
  DS_i^G[\xi] &= H^\top\!\diag\!\big(2\,Y_{i,:}\Had\Delta Y[\xi]_{i,:}\big)H \notag\\
  &\quad + \xi_H^\top\!\diag(Y_{i,:}^2)\,H + H^\top\!\diag(Y_{i,:}^2)\,\xi_H,
  \label{eq:app_DS}
\end{align}
with $S_j^H$, $S_i^U$, $S_j^V$ obtained under $Y\!\leftrightarrow\!X$ and $(G,H)\!\leftrightarrow\!(U,V)$.
\end{proposition}
\begin{proof}
On a vector space the metric is a field of symmetric positive-definite operators, so $g_x$ is nondegenerate and~\eqref{eq:app_koszul} determines $\Gamma_x(\xi,\eta)$ uniquely for every $\xi,\eta$. The derivative of a symmetric bilinear form is symmetric, so $(D_\zeta g)(\xi,\eta)=(D_\zeta g)(\eta,\xi)$. Subtracting~\eqref{eq:app_koszul} from its image under $\xi\!\leftrightarrow\!\eta$ and using this symmetry gives $g_x(\Gamma_x(\xi,\eta)-\Gamma_x(\eta,\xi),\zeta)=0$ for all $\zeta$, so $\Gamma_x$ is symmetric and $\nabla$ is torsion-free. Adding~\eqref{eq:app_koszul} to its image under $\eta\!\leftrightarrow\!\zeta$ gives $g_x(\Gamma_x(\xi,\eta),\zeta)+g_x(\eta,\Gamma_x(\xi,\zeta))=(D_\xi g)(\eta,\zeta)$, which equals $\xi\big[g_x(\eta,\zeta)\big]$ for constant fields, so $\nabla$ is metric-compatible. A torsion-free, metric-compatible connection is the Levi--Civita connection and is unique. Finally,~\eqref{eq:app_DS} is the directional derivative of $S_i^G=H^\top\diag(Y_{i,:}^2)H$ in~\eqref{eq:sblocks_GH}, obtained by varying $Y$ through $\Delta Y[\xi]$ and $H$ through $\xi_H$, and $(D_\xi g)(\eta,\zeta)=\sum_b\eta_b(DS_b[\xi])\zeta_b^\top$ follows by differentiating~\eqref{eq:metric} term by term.
\end{proof}

Because $g_x$ is block-diagonal, applying $g^{-1}$ in~\eqref{eq:app_gamma} is the same batched $r\times r$ solve as the gradient~\eqref{eq:rgrad_G}, and the first two terms are contractions of~\eqref{eq:app_DS}. The reaction covector collects the dependence of every weight block on the differentiated factor. Its $G$-block is
\begin{align}
  c_G(\xi,\eta) &= \big(Y^2\Had(G\eta_H^\top)\big)\,\xi_H + \big(Y^2\Had(G\xi_H^\top)\big)\,\eta_H \notag\\
  &\quad + 2\,(X\Had M)\,H,
  \label{eq:app_c}
\end{align}
with $M = (\xi_U V^\top)\Had(\eta_U V^\top) + (U\xi_V^\top)\Had(U\eta_V^\top)$, and $c_H$, $c_U$, $c_V$ obtained under $Y\!\leftrightarrow\!X$ and $(G,H)\!\leftrightarrow\!(U,V)$. Every term evaluates through matrix products, so no dense connection tensor is ever stored.

\subsection{Ingredient 3: the Riemannian Hessian}
The Hessian is the covariant derivative of the gradient. Setting $\eta=\mathrm{grad}\,f$ from~\eqref{eq:rgrad_G}--\eqref{eq:rgrad_H} in Ingredient~2 gives
\begin{equation}
  \mathrm{Hess}\,f(x)[\xi] = P_x^{\cH}\big(\nabla_\xi\,\mathrm{grad}\,f\big) = P_x^{\cH}\big(D\eta[\xi] + \Gamma_x(\xi,\eta)\big),
  \label{eq:app_hess_split}
\end{equation}
which is valid because $f$ is constant on the fibers of the submersion~\cite{absilbook,boumal2023introduction}. Two quantities are therefore needed beyond the projection. The first is the Christoffel term $\Gamma_x(\xi,\eta)$ of~\eqref{eq:app_gamma}. The second is the directional derivative $D\eta[\xi]$ of the Riemannian gradient, and it reduces further. Writing $\eta=g^{-1}\nabla f$ and using $D(g^{-1})[\xi]=-g^{-1}(D_\xi g)g^{-1}$,
\begin{equation}
  D\eta[\xi] = g^{-1}\big(H_E[\xi] - (D_\xi g)\eta\big),
  \label{eq:app_dgrad}
\end{equation}
where $(D_\xi g)\eta$ is the contraction of~\eqref{eq:app_DS} that $\Gamma_x$ already requires, and $H_E[\xi]=D(\nabla f)[\xi]$ is the Euclidean Hessian-vector product, the directional derivative of the factor gradients~\eqref{eq:egrad}. With $\dot W = D\Phi_x[\xi]$ from~\eqref{eq:dphi} and $\dot E = \nabla^2_W f[\dot W]$,
\begin{align*}
  H_E[\xi] = \big(&\dot E_Y H + E_Y\xi_H,\;\ \dot E_Y^\top G + E_Y^\top\xi_G, \\
  &\dot E_X V + E_X\xi_V,\;\ \dot E_X^\top U + E_X^\top\xi_U\big),
\end{align*}
where $\dot E_Y = \dot E\Had Y + E\Had\Delta Y[\xi]$ and $\dot E_X = \dot E\Had X + E\Had\Delta X[\xi]$. Only $\dot E$, through the loss curvature $\nabla^2_W f$, depends on the loss. For least squares $\nabla^2_W f=\mathrm{Id}$, so $\dot E=\dot W$.

\begin{proposition}[Riemannian Hessian]\label{prop:hess}
With $\eta=\mathrm{grad}\,f$, the Riemannian Hessian on the quotient is
\begin{align}
  \mathrm{Hess}\,f(x)[\xi] &= P_x^{\cH}\,g^{-1}\!\Big(H_E[\xi] \notag\\
  &\quad + \tfrac12\big[(D_\eta g)\xi - (D_\xi g)\eta - c(\xi,\eta)\big]\Big),
  \label{eq:app_hess}
\end{align}
where $c(\xi,\eta)=\nabla_x\!\big[g_x(\xi,\eta)\big]$ is the reaction covector~\eqref{eq:app_c}.
\end{proposition}
\begin{proof}
Substituting~\eqref{eq:app_dgrad} and~\eqref{eq:app_gamma} into~\eqref{eq:app_hess_split}, the terms $-g^{-1}(D_\xi g)\eta$ and $\tfrac12 g^{-1}(D_\xi g)\eta$ combine to $-\tfrac12 g^{-1}(D_\xi g)\eta$, which gives~\eqref{eq:app_hess}.
\end{proof}

Applying $g^{-1}$ is the batched $r\times r$ solve of the gradient~\eqref{eq:rgrad_G}, the three metric-derivative terms are matrix-product contractions, and $P_x^{\cH}$ is Ingredient~1 (Proposition~\ref{prop:hproj}). 


\subsection{Trust-region subproblem}
At each iterate RTR minimizes the model $g_x(\mathrm{grad}\,f,\xi) + \tfrac12 g_x(\mathrm{Hess}\,f[\xi],\xi)$ over horizontal $\xi$ in the trust ball $\|\xi\|_x\le\Delta$ with truncated conjugate gradient, and updates the factors with the retraction~\eqref{eq:retraction}. We use the trust-region optimizer of Pymanopt~\cite{pymanopt}, supplying the metric~\eqref{eq:metric}, gradient~\eqref{eq:rgrad_G}--\eqref{eq:rgrad_H}, Hessian~\eqref{eq:app_hess}, projection (Proposition~\ref{prop:hproj}), and retraction~\eqref{eq:retraction}.



\section{Comparison with the Concurrent Method of \citet{gillis2026manifold}}
\label{app:gillis}
Concurrently with this work, \citet{gillis2026manifold} proposed Riemannian algorithms for the same Hadamard decomposition. We describe their method and contrast the resulting expressions and per-iteration cost with ours. The head-to-head least-squares comparison appears in Section~\ref{sec:exp_secondorder}. The two works share the problem but not the geometry, and they differ in the losses and data regimes they address.

\begin{table*}[t]
\centering
\small
\caption{Distinction between the concurrent method of \citet{gillis2026manifold} and the present work.}
\label{tab:gillis_distinction}
\begin{tabular}{@{}lll@{}}
\toprule
Aspect & \citet{gillis2026manifold} & This work \\
\midrule
Variables    & two dense low-rank matrices $X_1,X_2$ & four factors $G,H,U,V$ \\
Geometry     & embedded fixed-rank & quotient by GL and cross-scaling (Sec.~\ref{sec:geometry}) \\
Metric       & standard Euclidean & preconditioned block-diagonal~\eqref{eq:metric} \\
Gradient     & projected ambient $-R\Had X_2$ & per-row solve $(\nabla_G f)_{i:}(S_i^G)^{-1}$~\eqref{eq:rgrad_G} \\
Hessian      & projected ambient & folded connection $P_x^{\cH}g^{-1}(H_E+\tfrac12[\,\cdots])$~\eqref{eq:app_hess} \\
Loss         & least squares only & any smooth loss \\
Data regime  & full matrix (dense or sparse input) & full matrix and completion with missing entries \\
Order        & second only & first (RGD) and second (RTR) \\
\bottomrule
\end{tabular}
\end{table*}

\subsection{The Gillis baseline}
\citet{gillis2026manifold} optimize the two dense low-rank matrices $X_1$ and $X_2$ in the representation $X\approx X_1\Had X_2$, each on the embedded fixed-rank manifold with the standard Euclidean metric, under the least-squares objective $\tfrac12\|X-X_1\Had X_2\|_F^2$. Their second-order method, which they call ``Manopt'', runs a truncated conjugate-gradient trust region with the projected ambient gradient $\Pi_{X_1}({-}R\Had X_2)$, where $R=X-X_1\Had X_2$ is the residual, and the projected ambient Hessian
\begin{equation*}
  \mathcal{H}[dX_1] = \Pi_{X_1}\!\big(dX_1\Had X_2^2 + Y\Had dX_2\big), \quad Y = 2X_1\Had X_2 - X,
\end{equation*}
where $\Pi_{X_1}$ projects onto the fixed-rank tangent space at $X_1$ and the block for $X_2$ is symmetric. They initialize with a face-splitting (FS) scheme. For large \emph{sparse but fully observed} matrices, such as document term-frequency matrices, they use separate block-coordinate algorithms. Their formulation targets least squares only and does not address matrix completion with missing entries or non-quadratic losses.

\subsection{Distinction from our approach}
Table~\ref{tab:gillis_distinction} summarizes the differences. Where they optimize two dense matrices on the embedded fixed-rank geometry, we optimize the four factors $(G,H,U,V)$ on the quotient of Section~\ref{sec:geometry} under the preconditioned block-diagonal metric~\eqref{eq:metric}. The two constructions yield different gradients and Hessians. Their gradient is a projected ambient Hadamard product, whereas ours is a per-row solve $(\mathrm{grad}_G f)_{i:}=(\nabla_G f)_{i:}(S_i^G)^{-1}$~\eqref{eq:rgrad_G} against the weight blocks~\eqref{eq:sblocks_GH}--\eqref{eq:sblocks_UV}. Their Hessian is a projected ambient expression, whereas ours folds the ambient Levi--Civita connection into $P_x^{\cH}g^{-1}(H_E+\tfrac12[\,\cdots])$~\eqref{eq:app_hess}.

The two methods also have different per-iteration complexity. The gradient and Hessian of \citet{gillis2026manifold} reconstruct the dense $m\times n$ matrices at every evaluation, so their method costs $O(mnr)$ per step and cannot go below $O(mn)$ because the ambient gradient $-R\Had X_2$ requires the full residual $R=X-X_1\Had X_2$. Our per-iteration cost is $O(Nr + (m{+}n)r^4)$ with $N=|\Omega|$ the number of observed entries (Table~\ref{tab:complexity}): a data-fit term of the same $O(Nr)$ order plus a preconditioner term (the metric weight blocks~\eqref{eq:sblocks_GH}--\eqref{eq:sblocks_UV} and, for the Hessian, the horizontal projection of Proposition~\ref{prop:hproj}) that is independent of $N$. In the completion regime $N=|\Omega|\ll mn$ our data-fit term drops to $O(|\Omega|r)$, whereas their dense residual keeps their method at $O(mn)$ per step, which provides a speedup for our appoach.




\end{document}